\documentclass[lettersize,journal]{IEEEtran}

\usepackage{cite}

\usepackage{amsmath}
\usepackage{amsthm}
\newtheoremstyle{colon}%
{}
{}
{\itshape}
{}
{\bfseries}
{:}
{ }
{}

\theoremstyle{colon}
	\newtheorem{theorem}{Theorem}

\usepackage{amssymb}
\usepackage{mathtools}

\usepackage{xpatch}
\makeatletter
\xpatchcmd{\proof}{\@addpunct{.}}{\@addpunct{:}}{}{}
\makeatother

\usepackage{makecell}
\usepackage{diagbox}
\usepackage{adjustbox}
\usepackage{multirow}
\usepackage[colorlinks=true, allcolors=blue]{hyperref}

\usepackage{amssymb}
\usepackage{bm}

\newcommand*{\trans}{^{\mathsf{T}}}

\DeclareMathOperator*{\argmin}{arg\,min} 

\usepackage{array}
\usepackage[caption=false, font=normalsize, labelfont=sf,textfont=sf,subrefformat=parens,labelformat=parens]{subfig}
\usepackage{fixltx2e}
\usepackage{stfloats}
\usepackage{tikz}
\usepackage{circuitikz}
\usetikzlibrary{arrows,shapes,positioning}

\usepackage{xurl}

\usepackage[linesnumberedhidden,ruled]{algorithm2e}
\SetAlgoCaptionSeparator{:}
\usepackage{paralist}
\usepackage{multicol}
\usepackage{lipsum}
\usepackage{enumerate}

\usepackage{booktabs,tabularx}

\title{Imbalanced Data Clustering using Equilibrium K-Means}

\author{{Yudong~He}
}

\begin{document}
\maketitle

\begin{abstract}
Centroid-based clustering algorithms, such as hard K-means (HKM) and fuzzy K-means (FKM), have suffered from learning bias towards large clusters. Their centroids tend to be crowded in large clusters, compromising performance when the true underlying data groups vary in size (i.e., imbalanced data). To address this, we propose a new clustering objective function based on the Boltzmann operator, which introduces a novel centroid repulsion mechanism, where data points surrounding the centroids repel other centroids. Larger clusters repel more, effectively mitigating the issue of large cluster learning bias. The proposed new algorithm, called equilibrium K-means (EKM), is simple, alternating between two steps; resource-saving, with the same time and space complexity as FKM; and scalable to large datasets via batch learning. We substantially evaluate the performance of EKM on synthetic and real-world datasets. The results show that EKM performs competitively on balanced data and significantly outperforms benchmark algorithms on imbalanced data. Deep clustering experiments demonstrate that EKM is a better alternative to HKM and FKM on imbalanced data as more discriminative representation can be obtained. Additionally, we reformulate HKM, FKM, and EKM in a general form of gradient descent and demonstrate how this general form facilitates a uniform study of K-means algorithms.
\end{abstract}


\IEEEpeerreviewmaketitle

\section{Introduction}
\label{sect1}
Imbalanced data refers to the true underlying groups of data having different sizes, which is common in datasets of medical diagnosis, fraud detection, and anomaly detection. Imbalanced data poses a challenge for learning algorithms because these algorithms tend to be biased towards the majority group~\cite{krawczyk2016learning}. While there is a considerable amount of research on supervised learning (e.g., classification) from imbalanced data~\cite{he2009learning, ramyachitra2014imbalanced,tanha2020boosting}, unsupervised learning has not been as thoroughly explored, because the unknown cluster sizes make the task more difficult~\cite{lu2019self}. Methods like resampling and boosting frequently used in supervised learning cannot be applied in unsupervised learning due to the lack of labels.

Clustering is an important unsupervised learning task involving grouping data into clusters based on similarity. K-means (KM) is the most popular clustering technique, valued for its simplicity, scalability, and effectiveness with real datasets. It can also be used as an initialization method for more advanced clustering techniques, such as the Gaussian mixture model (GMM)~\cite{press2007gaussian,shireman2017examining}. KM starts with an initial set of centroids (cluster centers) and iteratively refines them to increase cluster compactness. The hard KM (HKM, or Lloyd's algorithm)~\cite{macqueen1967some,lloyd1982least} and fuzzy KM (FKM, or Bezdek's algorithm)~\cite{bezdek2013pattern} are the two most representative KM algorithms. 
HKM assigns a data point to only one cluster, while FKM assigns a data point to multiple clusters with varying degrees of membership.

HKM and FKM stimulate many subsequent research. For example, the possibilistic K-means (PKM)~\cite{krishnapuram1993possibilistic} was proposed with reformulated membership more meaningful in terms of typicality. Later, the possibilistic fuzzy K-means (PFKM)~\cite{pal2005possibilistic} was proposed to address the noise sensitivity defect of FKM and overcome the coincident clusters problem of PKM. FKM-$\sigma$~\cite{tsai2011fuzzy} was proposed to improve FKM's performance on data points with uneven variations or non-spherical shapes in individual clusters. Fuzzy local information K-means~\cite{krinidis2010robust} was designed to promote FKM's performance in image segmentation. Recently, feature-weighted PKM (FWPKM)~\cite{yang2020feature,yang2021collaborative} was proposed to give non-uniform importance to features. Research of combining K-means and kernel mechanisms~\cite{zhang2003clustering,huang2011multiple,tang2023knowledge} has garnered interest, which non-linearly mapped data from low-dimensional space to high-dimensional space through kernel functions to enhance data separability. Conversely, deep clustering, a modern technique that integrates deep neural networks (DNNs) and K-means (or other clustering algorithms), has been proposed to cluster high-dimensional data by mapping them to a low-dimensional space to overcome the curse of dimensionality~\cite{coates2012learning,yang2017towards,caron2018deep,fard2020deep}.

Although many variations derived from HKM and FKM have been developed to deal with different situations, most display degraded effectiveness in cases of imbalanced data. This is due to the so-called ``uniform effect" that causes the clusters generated by these algorithms to have similar sizes even when input data has highly varying group sizes~\cite{xiong2006k}. It is reported that FKM has a stronger uniform effect than HKM~\cite{zhou2020effect}. An illustration of the uniform effect of HKM and FKM is given in Fig.~\ref{fig:ue} where we can observe that the centroids of HKM and FKM crowd together in the large cluster.

\begin{figure}[!t]
\centering
\subfloat[]{\includegraphics[width=0.24\textwidth]{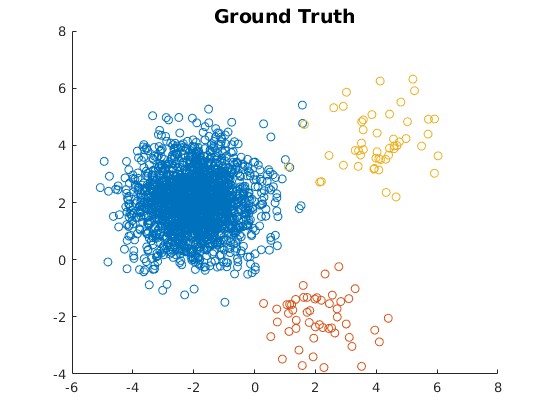}\label{fig:ue_dataset}}
 \subfloat[]{\includegraphics[width=0.24\textwidth]{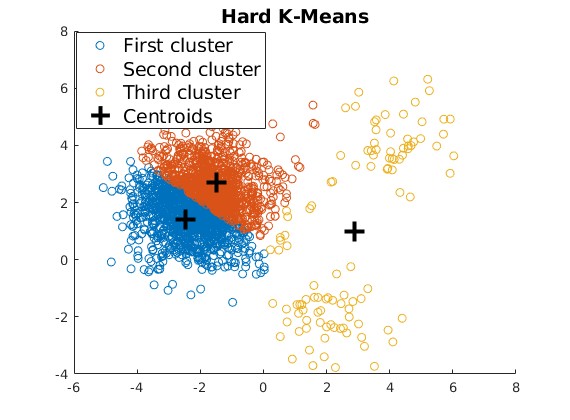}
 \label{fig:ue_km}}
 \hfill
 \subfloat[]{\includegraphics[width=0.24\textwidth]{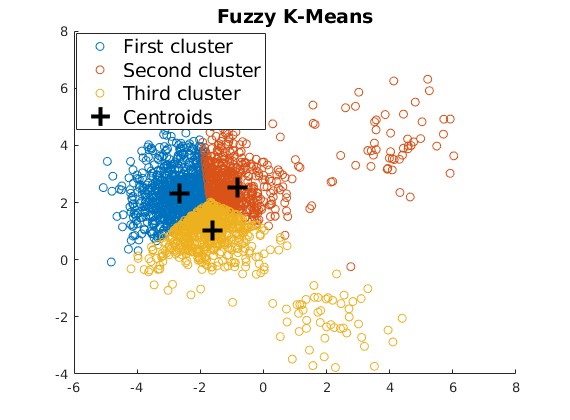}\label{fig:ue_fcm}}
\subfloat[]{\includegraphics[width=0.24\textwidth]{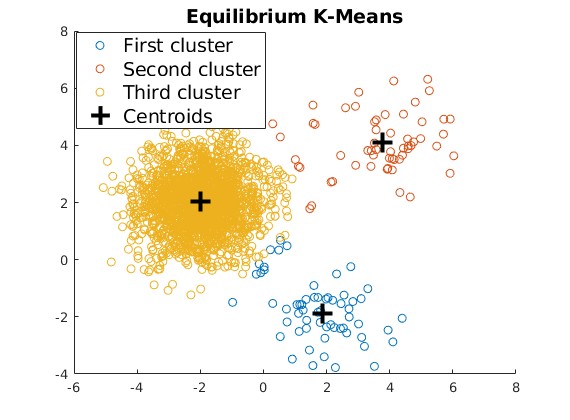}\label{fig:ue_ekm}}
 \caption{Clustering results of a highly imbalanced dataset. (a) Ground truth. The colors represent the reference labels of the data points. (b) Clustering by hard K-means. (c) Clustering by fuzzy K-means. (d) Clustering by the proposed equilibrium K-means. They are all two-step alternating algorithms that iteratively update centroids.}
\label{fig:ue}
\end{figure}

\subsection{Existing Efforts to Overcome Uniform Effect}
There are two popular methods to overcome the uniform effect. The first method is to introduce more weight on the data points in small clusters at each updating iteration, biasing learning towards them, like a modified FKM called cluster-size insensitive FKM (csiFKM) developed by Noordam et al.~\cite{noordam2002multivariate}. Later, Lin et al.~\cite{lin2014size} proposed a size-insensitive integrity-based FKM (siibFKM) based on csiFKM to reduce the sensitivity of csiFKM to the distance between adjacent clusters. However, weighting based on cluster size inadvertently increases the influence of outliers, making these algorithms sensitive to noise~\cite{askari2021fuzzy}. 

The second method is called multiprototype clustering. This method first groups data into multiple subclusters with similar sizes and the final clusters are obtained by merging adjacent subclusters. Liang et al.~\cite{liang2012k} proposed a multiprototype clustering algorithm that employs FKM to generate subclusters. Later, Lu et al.~\cite{lu2019self} proposed a self-adaptive multiprototype clustering algorithm that automatically adjusts the number of subclusters. However, multiprototype clustering algorithms have a complex process and high time complexity of $O(N^2)$, where $N$ is the number of data points in the dataset. Thus, they are computationally expensive for large datasets. We should additionally mention that Zeng et al.~\cite{zeng2023soft} recently proposed a soft multiprototype clustering algorithm with time complexity linear to $N$. However, their clustering process remains complex and is aimed at clustering high-dimensional and complex-structured data rather than imbalanced data.

\subsection{Our Contributions}
In this paper, we propose equilibrium K-means (EKM), which addresses the imbalanced data clustering issue by a new centroid repulsion mechanism without the above limitations. Our contributions can be summarized in three main aspects.

First, we reformulate HKM, FKM, the maximum-entropy fuzzy clustering (MEFC)~\cite{karayiannis1994meca,li1995maximum}, and EKM in a general form of gradient descent. We show that these algorithms aim to optimize different approximations of the same objective. This general form facilitates the uniform study of KM algorithms.

Second, we develop EKM based on the first contribution where the Boltzmann operator is utilized as an approximation method. Such approximation brings a brand-new centroid repulsion mechanism. Data surrounding centroids repel other centroids, and large clusters repel more as they contain more data, successfully reducing the uniform effect by preventing centroids from crowding together in a large cluster (see Fig.~\ref{fig:ue_ekm} for an example). Hence, EKM is robust to imbalanced data and not sensitive to noise. Similar to HKM and FKM, EKM is a two-step alternating algorithm that iteratively computes centroids. In addition to having the same time ($O(N)$) and space complexity as FKM, EKM has a batch-learning version that can be applied to large datasets. A comprehensive study is conducted using four synthetic and 16 real datasets (hence 20 datasets in total) to demonstrate the effectiveness of EKM on balanced and imbalanced data.

Finally, we investigate the combination of DNNs and EKM. We find that mapping high-dimensional data via DNNs to an EKM-friendly space can result in more discriminative low-dimensional representation than mapping to an HKM-friendly space. Compared to the combination of DNNs and HKM, using EKM improves clustering accuracy by 35\% on an imbalanced dataset derived from MNIST.

\subsection{Organization}
We introduce HKM, FKM, and MEFC in Section~\ref{sect2}. In Section~\ref{sect3}, we derive their general form and the proposed EKM. The properties of EKM and its centroid repulsion mechanism are studied in Section~\ref{sect4}. We evaluate the performance of EKM on classic clustering tasks in Section~\ref{sect5} and on deep clustering in Section~\ref{sect6}. Finally, we conclude in Section~\ref{sect7}.
\section{K-Means and Its Variations}
\label{sect2}
\subsection{The Hard K-Means Algorithm (Lloyd's Algorithm)}
KM aims to partition $N$ data points into $K$ clusters, minimizing the sum of the variances within each cluster. Mathematically, the objective can be expressed as:
\begin{equation}
\label{kmeans_obj}
     \underset{\mathbb{S}_1,\cdots,\mathbb{S}_K}{\arg\min} \sum_{k=1}^K \sum_{\mathbf{x}\in \mathbb{S}_k} \|\mathbf{x}-\boldsymbol{\mu}_k\|_2^2,
\end{equation}
where $\mathbb{S}_k$ represents the set of data points in the $k$-th cluster, $\mathbf{x}\in\mathbb{S}_k$ denotes a data point belonging to $\mathbb{S}_k$, $\boldsymbol{\mu}_k$ is the centroid of $\mathbb{S}_k$, expressed by 
\begin{equation}
    \boldsymbol{\mu}_k=\frac{1}{|\mathbb{S}_k|}\sum_{\mathbf{x}\in \mathbb{S}_k} \mathbf{x},
\end{equation}
$|\mathbb{S}_k|$ signifies the size of $\mathbb{S}_k$, and $\|\cdot\|_2$ is the $l_2$ norm. This optimization problem is NP-hard~\cite{aloise2009np} and is commonly solved heuristically by Lloyd‘s algorithm~\cite{lloyd1982least}. Given $K$ points ($\mathbf{c}_1^{(1)},\cdots,\mathbf{c}_K^{(1)})$ as initial centroids, Lloyd's algorithm alternates between two steps:
\begin{enumerate}
    \item Assign each data point to the nearest centroid, forming $K$ clusters:
    \begin{equation}
        \mathbb{S}^{(\tau)}_k=\{\mathbf{x}:\|\mathbf{x}-\mathbf{c}_k^{(\tau)}\|_2^2\le \|\mathbf{x}-\mathbf{c}_i^{(\tau)}\|_2^2 \,\forall i,\,1\le i \le K\}
    \end{equation}
    \item Recalculate the centroid of each cluster by taking the mean of all data points assigned to it:
    \begin{equation}
        \mathbf{c}_k^{(\tau+1)}=\frac{1}{|\mathbb{S}_k^{(\tau)}|}\sum_{\mathbf{x}\in \mathbb{S}_k^{(\tau)}} \mathbf{x}.
    \end{equation}
\end{enumerate}
Lloyd's algorithm converges when the assignment of data points to clusters ceases to change, or when a maximum number of iterations is reached. The time complexity of one iteration of the above two steps is $O(NK)$. The HKM algorithm mentioned in this paper refers to Lloyd's algorithm.

\subsection{The Uniform Effect of K-Means}
The uniform effect refers to the propensity to generate clusters of similar sizes.  It is essentially a learning bias towards large clusters and is implicitly implied in the objective~\eqref{kmeans_obj} of KM. For simplicity, kindly consider a case with two clusters (i.e., $K=2$). Minimizing the objective function~\eqref{kmeans_obj} is equivalent to maximizing the following objective function~\cite{wu2012advances}:
\begin{equation}
    \max_{\mathbb{S}_1,\,\mathbb{S}_2} N_1N_2\|\boldsymbol{\mu}_1-\boldsymbol{\mu}_2\|_2^2,
\end{equation}
where $N_1$ and $N_2$ denote the sizes of $\mathbb{S}_1$ and $\mathbb{S}_2$, respectively. By isolating the effect of $\|\boldsymbol{\mu}_1-\boldsymbol{\mu}_2\|_2^2$, maximizing the above objective leads to $N_1=N_2=N/2$, indicating KM tends to produce equally sized clusters.

\subsection{The Fuzzy K-Means Algorithm (Bezdek's Algorithm)}
FKM attempts to minimize the sum of weighted distances between data points and centroids, with the following objective and constraints~\cite{bezdek2013pattern}:
\begin{equation}
    \begin{aligned}
    \label{fcm obj}
    \min_{\mathbf{c}_1,\cdots,\mathbf{c}_K,\{u_{kn}\}}  \quad & \sum_{n=1}^N \sum_{k=1}^K (u_{kn})^m \|\mathbf{x}_n-\mathbf{c}_k\|_2^2 \\
    \textrm{subject to} \quad & u_{kn}\in [0,\,1]\,\forall k,n,\\
    & \sum_{k=1}^K u_{kn}=1\, \forall n, \\
    & 0<\sum_{n=1}^N u_{kn}<N\, \forall k, \\
    \end{aligned}
\end{equation}
where $\mathbf{x}_n$ represents the $n$-th data point, $\mathbf{c}_k$ denotes the $k$-th centroid, $u_{kn}$ is a coefficient called membership that indicates the degree of $\mathbf{x}_n$ belonging to the $k$-th cluster, and $m\in(1,+\infty)$ is a hyperparameter controlling the degree of fuzziness level. Similar to the HKM algorithm, the FKM algorithm operates by alternating two steps:
\begin{enumerate}
    \item Calculate the membership value of the $n$-th data point belonging to the $k$-th cluster:
    \begin{equation}
    \label{fcm_step1}
        u_{kn}^{(\tau)}=\frac{1}{\sum_{i=1}^K \bigg(\frac{\|\mathbf{x}_n-\mathbf{c}_k^{(\tau)}\|_2}{\|\mathbf{x}_n-\mathbf{c}_i^{(\tau)}\|_2}\bigg)^{\frac{2}{m-1}}}.
    \end{equation}

    \item Recalculate the weighted centroid of the $k$-th cluster by:
    \begin{equation}
    \label{fcm_step2}
        \mathbf{c}_k^{(\tau+1)}=\frac{\sum_n (u_{kn}^{(\tau)})^m \mathbf{x}_n}{\sum_n (u_{kn}^{(\tau)})^m}.
    \end{equation}
\end{enumerate}
The time complexity of one iteration of the above two steps is $O(NK^2)$. The higher time complexity of FKM than HKM is due to the extra membership calculation.

\subsection{Maximum-Entropy Fuzzy Clustering}
Karayiannis~\cite{karayiannis1994meca} added an entropy term to the objective function of FKM, resulting in MEFC. The new objective function is given as follows:
\begin{equation}
    \begin{aligned}
    \label{meca obj}
    \min_{\mathbf{c}_1,\cdots,\mathbf{c}_K,\{u_{kn}\}}  \quad & \eta\sum_{n=1}^N\sum_{k=1}^K u_{kn}\ln{u_{kn}}\\
    & +(1-\eta)\frac{1}{N}\sum_{n=1}^N \sum_{k=1}^K u_{kn}\|\mathbf{x}_n-\mathbf{c}_k\|_2^2 \\
    \textrm{subject to} \quad & u_{kn}\in [0,\,1]\,\forall k,n,\\
    & \sum_{k=1}^K u_{kn}=1\, \forall n, \\
    \end{aligned}
\end{equation}
where $\eta\in(0,1)$ is a hyperparameter, controlling the transition from maximization of the entropy to the minimization of centroid-data distances. MEFC is similar to FKM but with a different definition of membership:
\begin{equation}
\label{mefc_membership}
    u^{(\tau)}_{kn}=\frac{\exp{(-\lambda \|\mathbf{x}_n-\mathbf{c}_k^{(\tau)}\|_2^2)}}{\sum_{i=1}^K \exp{(-\lambda \|\mathbf{x}_n-\mathbf{c}_i^{(\tau)}\|_2^2)}},
\end{equation}
where $\lambda=\frac{1}{N} \frac{1-\eta}{\eta}$. The time complexity of MEFC is the same as FKM, which is $O(NK^2)$ for one iteration.



\section{Smooth K-Means - A Unified Framework}
\label{sect3}
\subsection{Objective of Smooth K-Means}
We propose a novel framework called smooth K-means (SKM) and demonstrate that the three KM algorithms introduced in the previous section are special cases of SKM.

Denote the squared Euclidean distance\footnote{Other distance metrics, such as the absolute difference and the angle between points, can also be used to define $d_{kn}$. Considering the derivation process is the same, we use the Euclidean distance in the paper for simplicity.} between the $k$-th centroid and the $n$-th data point as
\begin{equation}
     d_{kn}=\frac{1}{2}\|\mathbf{x}_n-\mathbf{c}_k\|_2^2,
\end{equation}
and define the within-cluster sum of squares (WCSS) as the sum of squared Euclidean distances between data points and their nearest centroids, i.e.,
\begin{equation}
\label{WCSS}
    \text{WCSS}:=\sum_{n=1}^N \min(d_{1n},\cdots,d_{Kn}).
\end{equation}
The goal of SKM is to find $K$ centroids that minimize an approximated WCSS, resulting in the following optimization problem:
\begin{equation}
     \underset{\mathbf{c}_1,\cdots,\mathbf{c}_K} \arg \min \sum_{n=1}^N h(d_{1n},\cdots,d_{Kn}),
\end{equation}
where $h(d_{1n},\cdots,d_{Kn})$ is a smooth approximation to $\min(d_{1n},\cdots,d_{Kn})$ and is referred to as the smooth minimum function. Below we reveal the relationship between the SKM and the GMM.

\subsection{The Relationship Between SKM and GMM}
The centroids obtained by minimizing WCSS are maximum likelihood estimators (MLEs) of parameters of ``hard" GMM. The derivation is as follows. Assuming that the dataset $(\mathbf{x}_1,\cdots,\mathbf{x}_N)$ is sampled from $K$ independent multivariate normal distributions. We denote the mean vector and covariance matrix of the $k$-th normal distribution as $\boldsymbol{\mu}_k$ and $\mathbf{\Sigma}_k$, respectively. The standard GMM has the following well-known likelihood
\begin{equation}
    \begin{aligned}
        \label{Likelihood_HGMM}
        L(\boldsymbol{\mu},\mathbf{\Sigma},p\,|\,\mathbf{x})\propto &\prod_{n=1}^N \sum_{k=1}^K  p_{kn}\det(\mathbf{\Sigma}_k)^{-1/2}\\
        &\exp{\bigl(-\frac{1}{2}(\mathbf{x}_n-\boldsymbol{\mu}_k)\trans \mathbf{\Sigma}_k^{-1} (\mathbf{x}_n-\boldsymbol{\mu}_k)\bigl)},
    \end{aligned}
\end{equation}
where $p_{kn}$ is the probability that the $n$-th data point is generated from the $k$-th normal distribution. The standard GMM assumes that data points are generated from multiple normal distributions based on a certain probability distribution. On the other hand, the hard GMM assumes that a data point is generated from a single normal distribution. Accordingly, for all $n$, $p_{kn}=1$ for one $k\in \{1,\cdots,N\}$ and $p_{in}=0$ if $i\ne k$. We further assume that all normal distributions have the same covariance matrix, specifically an identity matrix. Under these assumptions, the MLEs of $\{p_{kn}\}$ are $\hat{p}_{kn}=1$ if 
$k=\arg\min_{i\in \{1,\cdots,K\}} \|\mathbf{x}_n-\boldsymbol{\mu}_i\|_2^2$, and $\hat{p}_{kn}=0$ otherwise. By substituting $\mathbf{\Sigma}_k=\mathbf{I}$ and $\hat{p}_{kn}$ into~\eqref{Likelihood_HGMM} and taking the logarithmic value, we obtain the log-likelihood:
\begin{equation}
    l(\boldsymbol{\mu}\,|\,\mathbf{x})\propto -\sum_{n=1}^N \min(\|\mathbf{x}_n-\boldsymbol{\mu}_1\|_2^2,\cdots,\|\mathbf{x}_n-\boldsymbol{\mu}_K\|_2^2).
\end{equation}
Now it is clear that the MLEs of $\{\boldsymbol{\mu}_1,\cdots,\boldsymbol{\mu}_K\}$ are the centroids minimizing WCSS. 

The hard GMM model simplifies the standard GMM by only considering the impact of a data point on its closest centroid. However, as seen in the success of FKM, it can be more advantageous to consider the impact of a data point on all centroids. This can be accomplished in SKM by applying an approximation function to smooth WCSS. Hence, SKM can be viewed as a model between hard and standard GMM. In the following sections, we will introduce three common smooth minimum functions and explore the resulting distinct clustering algorithms.

\subsection{Three Common Smooth Minimum Functions}
Assume a monotonically increasing and differentiable function $f: [0,\,+\infty) \mapsto [0,\,+\infty)$ satisfies
\begin{equation}
    \lim_{x\to +\infty} \frac{x}{f(x)}\to 0,
\end{equation}
or equivalently
\begin{equation}
    \lim_{x\to +\infty} \frac{1}{f'(x)}\to 0,
\end{equation}
where $f'$ is the first derivative of $f$. Let $g(x)=1/f(x)$, a smooth minimum function $h_1: [0,\,+\infty)^K \mapsto [0,+\infty)$ can be constructed by 
\begin{equation}
    h_1(x_1,\cdots,x_K)=g^{-1}(g(x_1)+\cdots+g(x_K)).
\end{equation}

Let $f(x)$ be defined as $e^{\lambda x}$, then the function $h_1$ takes on a specific form known as LogSumExp:
\begin{equation}
\label{LSE}
\begin{aligned}    h_1(x_1,\cdots,x_K)&=\text{LSE}_\lambda(x_1,\cdots,x_K)\\
    &=-\frac{1}{\lambda}\ln(e^{-\lambda x_1}+\cdots+e^{-\lambda x_K}),
\end{aligned}
\end{equation}
where $\lambda$ is a parameter controlling the degree of approximation, with $\text{LSE}_\lambda \to \min$ as $\lambda \to +\infty$.

Define $f(x)=x^p$, we can have another common smooth minimum function, called p-Norm, which has the following specific form
\begin{equation}
\label{PNorm}
\begin{aligned}
    h_1(x_1,\cdots,x_K)&=\text{PN}_p(x_1,\cdots,x_K)\\
    &=(x_1^{-p}+\cdots+x_K^{-p})^{-1/p},
\end{aligned}
\end{equation}
and converges to $\min(x_1,\cdots,x_K)$ as $p\to +\infty$. An additional definition is necessary for the domain of the p-Norm function to legally include zeros: $\text{PN}_p(x_1,\cdots,x_K)=0$ if some $x_i=0$ where $1\le i \le K$. 

Smooth minimum functions can also be constructed by
\begin{equation}
    h_2(x_1,\cdots,x_K)=\frac{x_1 g(x_1)+\cdots+x_K g(x_K)}{g(x_1)+\cdots+g(x_K)}.
\end{equation}
A specific example is the Boltzmann operator, where $g(x)=e^{-\alpha x}$. The Boltzmann operator takes on the form of
\begin{equation}
\label{Boltzmann}
    \begin{aligned}
        h_2(x_1,\cdots,x_K)&=\text{boltz}_\alpha (x_1,\cdots,x_K)\\
        &= \frac{\sum_{i=1}^K x_i e^{-\alpha x_i}}{\sum_{i=1}^K e^{-\alpha x_i}},
    \end{aligned}
\end{equation}
and converges to the minimum function as $\alpha\to +\infty$.

\subsection{Clustering by Minimizing Smoothed WCSS}
\label{sect3-D}
\subsubsection{The relationship between LogSumExp and Maximum-Entropy Fuzzy Clustering}
Approximating WCSS~\eqref{WCSS} by LogSumExp~\eqref{LSE}, we have the following objective:
\begin{equation}
\label{LSE objective}
    \min_{\mathbf{c}_1,\cdots,\mathbf{c}_K} J_{\text{LSE}}(\mathbf{c}_1,\cdots,\mathbf{c}_K)=\sum_{n=1}^N
    -\frac{1}{\lambda} \ln{(\sum_{k=1}^K e^{-\lambda d_{kn}})}.
\end{equation}

The differentiable objective function $J_{\text{LSE}}$ has the first-order partial derivative:
\begin{equation}
\begin{aligned}
   \partial_{\mathbf{c}_k}{J_{\text{LSE}}}=\frac{\partial{J_{\text{LSE}}}}{\partial{\mathbf{c}_k}}&=\sum_{n=1}^N \frac{e^{-\lambda d_{kn}}}{\sum_{i=1}^K e^{-\lambda d_{in}}}\frac{\partial{d_{kn}}}{\partial{\mathbf{c}_k}}\\
   &=-\sum_{n=1}^N \frac{e^{-\lambda d_{kn}}}{\sum_{i=1}^K e^{-\lambda d_{in}}}(\mathbf{x}_n-\mathbf{c}_k).
\end{aligned}
\end{equation}

The minimizer of $J_{\text{LSE}}$ can be found by gradient descent iteration:
\begin{equation}
\label{LSE-gradient-iteration}
    \mathbf{c}_k^{(\tau+1)}=\mathbf{c}_k^{(\tau)}-\gamma^{(\tau)}_k  \partial_{\mathbf{c}_k}{J_{\text{LSE}}}(\mathbf{c}_k^{(\tau)}),
\end{equation}
where $\gamma^{(\tau)}_k$ is the learning rate at the $\tau$-th iteration. This updating procedure~\eqref{LSE-gradient-iteration} is equivalent\footnote{It should be noted that the equivalence mentioned in this paper is in the sense of the algorithm level, not in the criterion level.} to the MEFC algorithm if one set 
\begin{equation}
\label{learning_rate_LSE}
    \gamma^{(\tau)}_k=1/\sum_{n=1}^N\bigg(\frac{e^{-\lambda d_{kn}^{(\tau)}}}{\sum_{i=1}^K e^{-\lambda d_{in}^{(\tau)}}}\bigg).
\end{equation}
Such learning rate value is related to the second-order partial derivative of $J_{\text{LSE}}$, which we will discuss later. The membership~\eqref{mefc_membership} of MEFC is identical to $\partial J_\text{LSE}/\partial d_{kn}$.

\subsubsection{Towards Lloyd's algorithm}
In the limit of $\lambda\to\infty$, $\frac{e^{-\lambda d_{kn}}}{\sum_i e^{-\lambda d_{in}}}\to 1$ if $d_{kn}\le d_{in} \, \forall i\in\{1,\cdots,K\}$, and $\frac{e^{-\lambda d_{kn}}}{\sum_i e^{-\lambda d_{in}}}\to 0$ otherwise\footnote{We assume $\forall n$, if $k\ne i$, then $d_{kn}\ne d_{in}$.}. In this case, the learning rate $\gamma_k^{(\tau)}\to 1/N_k^{(\tau)}$ where $N_k^{(\tau)}$ is the number of data points closest to $\mathbf{c}_k^{(\tau)}$, and the updating procedure~\eqref{LSE-gradient-iteration} approaches to Lloyd's algorithm.

\subsubsection{Lloyd's algorithm and Newton's method}
The smooth objective function $J_{\text{LSE}}$ possesses the second-order partial derivative given by
\begin{equation}
\label{second_partial_LSE}
\begin{aligned}
\partial_{\mathbf{c}_k}^2{J_{\text{LSE}}}=\frac{\partial^2{J_{\text{LSE}}}}{\partial{\mathbf{c}_k}^2}
    &=-\sum_{n=1}^N \bigg(
     \frac{e^{-\lambda d_{kn}}}{\sum_i e^{-\lambda d_{in}}}\frac{\partial^2 d_{kn}}{\partial \mathbf{c}_k^2}\\
     &+\frac{-\lambda e^{-\lambda d_{kq}}\sum_{i\ne k}e^{-\lambda d_{in}}}{(\sum_i e^{-\lambda d_{in}})^2}\frac{\partial d_{kn}}{\partial \mathbf{c}_k}(\frac{\partial d_{kn}}{\partial \mathbf{c}_k})\trans
     \bigg)
     \\
     &=\sum_{n=1}^N \frac{e^{-\lambda d_{kn}}}{\sum_i e^{-\lambda d_{in}}}(\mathbf{I}+\xi_{kn}\mathbf{D}_{kn})
\end{aligned}
\end{equation}
where $\xi_{kn}=\frac{\lambda \sum_{i\ne k}e^{-\lambda d_{in}}}{\sum_i e^{-\lambda d_{in}}}$ and $\mathbf{D}_{kn}=(\mathbf{x}_n-\mathbf{c}_k)(\mathbf{x}_n-\mathbf{c}_k)\trans$. By employing Newton's method, the minimizer of $J_{\text{LSE}}$ can be iteratively found using
\begin{equation}
\label{LSE-Newton-iteration}
    \mathbf{c}_k^{(\tau+1)}=\mathbf{c}_k^{(\tau)}- [\partial_{\mathbf{c}_k}^2{J_{\text{LSE}}}(\mathbf{c}_k^{(\tau)})]^{-1}  \partial_{\mathbf{c}_k}{J_{\text{LSE}}}(\mathbf{c}_k^{(\tau)}).
\end{equation}
As $\lambda \to +\infty$, $\partial_{\mathbf{c}_k}^2{J_{\text{LSE}}(\mathbf{c}_k^{(\tau)})}\to N_k^{(\tau)} \mathbf{I}$, and the updating procedure~\eqref{LSE-Newton-iteration} approaches Lloyd's algorithm. This indicates that Lloyd's algorithm is essentially Newton's method, which aligns with the perspective presented in~\cite{bottou1994convergence}. There is a close relationship between the learning rate value~\eqref{learning_rate_LSE} and the second-order partial derivative of $J_\text{LSE}$~\eqref{second_partial_LSE}: the gradient descent with the learning rate of ~\eqref{learning_rate_LSE} is equivalent to an approximated Newton's method in the sense that one term (the rank-1 matrix $\mathbf{D}_{kn}$) in~\eqref{second_partial_LSE} is ignored. A discussion of the pros and cons of optimizing $J_\text{LSE}$ using gradient descent or Newton's method would be valuable, but it is beyond the scope of this paper. In the following section, we will observe the same connection between Newton's method and FKM.

\subsubsection{The relationship between p-Norm and fuzzy K-means}
By substituting the minimum function in WCSS with the $p$-Norm function~\eqref{PNorm}, we have
\begin{equation}
\label{PNorm objective}
    \min_{\mathbf{c}_1,\cdots,\mathbf{c}_K} J_{\text{PN}}(\mathbf{c}_1,\cdots,\mathbf{c}_K)=\sum_{n=1}^N (d_{1n}^{-p}+\cdots+d_{Kn}^{-p})^{-1/p}.
\end{equation}

The objective $J_\text{PN}$ has the first partial derivative of
\begin{equation}
    \partial_{\mathbf{c}_k}{J_{\text{PN}}}=\frac{\partial{J_{\text{PN}}}}{\partial{\mathbf{c}_k}}=-\sum_{n=1}^N \frac{d_{kn}^{-p-1}}{(\sum_{i=1}^K d_{in}^{-p})^{1/p+1}}(\mathbf{x}_n-\mathbf{c}_k).
\end{equation}
The minimizer of $J_\text{PN}$ can be located using gradient descent iteration
\begin{equation}
\label{PN-gradient-iteration}
    \mathbf{c}_k^{(\tau+1)}=\mathbf{c}_k^{(\tau)}-\gamma^{(\tau)}_k  \partial_{\mathbf{c}_k}{J_{\text{PN}}}(\mathbf{c}_k^{(\tau)}).
\end{equation}
When setting $p=1/(m-1)$ and
\begin{equation}
    \gamma^{(\tau)}_k=1/\sum_{n=1}^N \frac{(d_{kn}^{(\tau)})^{-p-1}}{(\sum_{i=1}^K (d_{in}^{(\tau)})^{-p})^{1/p+1}},
\end{equation}
this gradient descent iteration is equivalent to the FKM algorithm. The membership~\eqref{fcm_step1} of FKM is equivalent to a power exponent ($1/m$) of $\partial J_\text{PN}/\partial d_{kn}$.

The connection between FKM and Newton's method becomes evident by taking the second-order partial derivative of $J_\text{PN}$, which is
\begin{equation}
    \partial^2_{\mathbf{c}_k}{J_{\text{PN}}}=\frac{\partial^2{J_{\text{PN}}}}{\partial{\mathbf{c}_k}^2}=\sum_{n=1}^N \frac{d_{kn}^{-p-1}}{(\sum_i d_{in}^{-p})^{1/p+1}}(\mathbf{I}+\zeta_{kn}\mathbf{D}_{kn}),
\end{equation}
where $\zeta_{kn}=(p+1)\frac{\sum_{i\ne k}d_{in}^{-p}}{d_{kn}\sum_i d_{in}^{-p}}$.  Hence, FKM can also be viewed as an approximated Newton's method that ignores the term of the rank-1 matrix $\mathbf{D}_{kn}$.
\subsection{From Boltzmann Operator to A Novel Clustering Algorithm}
\label{sect3-E}
Employing the Boltzmann operator to smooth WCSS results in
\begin{equation}
\label{Boltzmann objective}
    \min_{\mathbf{c}_1,\cdots,\mathbf{c}_K} J_{\text{B}}(\mathbf{c}_1,\cdots,\mathbf{c}_K)=\sum_{n=1}^N \frac{\sum_{i=1}^K d_{in} e^{-\alpha d_{in}}}{\sum_{i=1}^K e^{-\alpha d_{in}}}.
\end{equation}

The objective $J_B$ possesses the first-order partial derivative of
\begin{equation}
\begin{aligned}
    \partial_{\mathbf{c}_k} J_\text{B}=\frac{\partial J_{\text{B}}}{\partial \mathbf{c}_k}=-\sum_{n=1}^N & \frac{e^{-\alpha d_{kn}}}{\sum_i e^{-\alpha d_{in}}} \big[1-\alpha(d_{kn}\\
    &-\frac{\sum_i d_{in}e^{-\alpha d_{in}}}{\sum_i e^{-\alpha d_{in}}})\big](\mathbf{x}_n-\mathbf{c}_k).
\end{aligned}
\end{equation}
The minimizer of $J_B$ can be found using gradient descent iteration
\begin{equation}
    \mathbf{c}_k^{(\tau+1)}=\mathbf{c}_k^{(\tau)}-\gamma^{(\tau)}_k \partial_{\mathbf{c}_k} J_{\text{B}}(\mathbf{c}_1^{(\tau)},\cdots,\mathbf{c}_K^{(\tau)}),
\end{equation}
where 
\begin{equation}
    \gamma^{(\tau)}_k=1/\bigl(\sum_{n=1}^N \frac{e^{-\alpha d_{kn}^{(\tau)}}}{\sum_i e^{-\alpha d_{in}^{(\tau)}}} [1-\alpha(d_{kn}^{(\tau)}-\frac{\sum_i d_{in}^{(\tau)}e^{-\alpha d_{in}^{(\tau)}}}{\sum_i e^{-\alpha d_{in}^{(\tau)}}})]\bigl).
\end{equation}
This updating procedure which we call EKM can be reformulated into a two-step iteration procedure akin to FKM:

\begin{enumerate}
    \item Calculate the weight value of the $n$-th data point to the $k$-th cluster by
    \begin{equation}
    \label{ekm_step1}
        w_{kn}^{(\tau)}=\frac{e^{-\alpha d_{kn}^{(\tau)}}}{\sum_{i=1}^K e^{-\alpha d_{in}^{(\tau)}}} [1-\alpha(d_{kn}^{(\tau)}-\frac{\sum_{i=1}^K d_{in}^{(\tau)}e^{-\alpha d_{in}^{(\tau)}}}{\sum_{i=1}^K e^{-\alpha d_{in}^{(\tau)}}})].
    \end{equation}

    \item Recalculate the weighted centroid of the $k$-th cluster by
    \begin{equation}
    \label{ekm_step2}
        \mathbf{c}_k^{(\tau+1)}=\frac{\sum_n w_{kn}^{(\tau)} \mathbf{x}_n}{\sum_n w_{kn}^{(\tau)}}.
    \end{equation}
\end{enumerate}
EKM converges when centroids cease to change or the maximum number of iterations is reached. The time complexity of one iteration of the above two steps is $O(NK^2)$. For convenience, we summarise the complete procedure of EKM in Algorithm~\ref{alg:ekm}.
\begin{algorithm}
    \caption{Equilibrium K-Means Algorithm}          
    \label{alg:ekm}   
    \KwIn{
        A dataset $X=\{\mathbf{x}_n\}_{n=1}^N$, cluster number $K$, initial centroids $\{\mathbf{c}_k^{(0)}\}_{k=1}^K$,
        smoothing parameter $\alpha$}
    \KwOut {Centroids $\{\mathbf{c}_1,\cdots,\mathbf{c}_K\}$}

    $\tau=0$\;
   \Repeat{convergence}
    {
    Compute weight $w_{kn}^{(\tau)}$ by~\eqref{ekm_step1} for all $k,\,n$\;
    Update centroid $\mathbf{c}_k^{(\tau+1)}$ by~\eqref{ekm_step2} for all $k$\;
    $\tau=\tau+1$\;
    }
    \Return{$\{\mathbf{c}_k^{(\tau)}\}_{k=1}^K$}
    \end{algorithm}
\subsection{Physical Interpretation of Equilibrium K-Means}
The second law of thermodynamics asserts that in a closed system with constant external parameters (e.g., volume) and fixed entropy, the internal energy will reach its minimum value at the state of thermal equilibrium. The objective~\eqref{Boltzmann objective} of EKM follows this minimum energy principle. This connection can be established by envisioning data points as particles with discrete/quantized energy levels, where the number of energy levels is equivalent to the number of centroids, and the energy value corresponds to the squared Euclidean distance between a data point and a centroid. 

Boltzmann’s law tells that at the state of thermodynamic equilibrium, the probability of a particle occupying a specific energy level decreases exponentially with the increase of the energy value of that level. Hence, the objective function~\eqref{Boltzmann objective} equals the expectation of the entire system's energy, and EKM seeks centroids to minimize this energy expectation. Due to this connection, we refer to the proposed Algorithm~\ref{alg:ekm} as equilibrium K-means.

\subsection{Membership Defined in Equilibrium K-Means}
\label{sect3-G}
Membership can be defined in EKM. Although the sum of weights~\eqref{ekm_step1} of EKM equals one, i.e.,
\begin{equation}
    \sum_{k=1}^K w_{kn}^{(\tau)}=1,\, \forall n,
\end{equation}
it is worth noting that these weights cannot be interpreted as probabilities or memberships since some weight values are negative. We define membership of EKM from its physical perspective. From the physical interpretation of EKM, the exponential term $e^{-\alpha d_{kn}}$ can be interpreted as the unnormalized probability of the $n$-th data point belonging to the $k$-th clusters. Hence, the membership of the $n$-th data point to the $k$-th cluster can be defined as
\begin{equation}
    u_{kn} = \frac{e^{-\alpha d_{kn}}}{\sum_{i=1}^K e^{-\alpha d_{in}}}.
\end{equation}
Since we have no intention in this paper to discuss the membership (it is just an intermediate product), we only define the membership in EKM and leave further discussion to subsequent research.

\subsection{Convergence of Smooth K-Means}
HKM, FKM, MEFC, and EKM are special cases of SKM, which can be generalized as the following gradient descent algorithm:
\begin{equation}
\label{smooth_km_sgd}
    \mathbf{c}_k^{(\tau+1)}=\mathbf{c}_k^{(\tau)}-\gamma^{(\tau)}_k \partial_{\mathbf{c}_k} J(\mathbf{c}_1^{(\tau)},\cdots,\mathbf{c}_K^{(\tau)}),
\end{equation}
where $J(\mathbf{c}_1^{(\tau)},\cdots,\mathbf{c}_K^{(\tau)})=\sum_{n=1}^N h(d_{1n}^{(\tau)},\cdots,d_{Kn}^{(\tau)})$, $h$ is a smooth minimum function, $d_{kn}^{(\tau)}=\frac{1}{2}\|\mathbf{x}_n-\mathbf{c}_k^{(\tau)}\|_2^2$, and the learning rate $\gamma^{(\tau)}_k$ is given by
\begin{equation}
    \gamma^{(\tau)}_k=1/\sum_{n=1}^N \bigg( \frac{\partial h}{\partial d_{kn}}(d_{1n}^{(\tau)},\cdots,d_{Kn}^{(\tau)})\bigg).
\end{equation}
Different KM algorithms can be obtained by taking $h$ the corresponding explicit form (refer to Section~\ref{sect3-D} to \ref{sect3-E}). This general form facilitates the uniform study of KM algorithms. Below we give a convergence guarantee conditioning on the properties of $h$:
\begin{theorem}[Convergence Condition]
\label{theorem_convergence}
    The centroid sequence obtained by~\eqref{smooth_km_sgd} converges to a (local) minimizer or saddle point of the objective function $J$ if the following conditions can be satisfied:
    \begin{enumerate}
        \item (Concavity) The function $h$ is a concave function at its domain $[0,\,+\infty)^K$.
        \item (Boundness) The function $h$ has a lower bound, i.e., $h>-\infty$, and the learning rate set $\{\gamma^{(\tau)}_k\}_{\tau,k}$ has a positive lower bound, i.e., $\exists \epsilon>0$, such that $\gamma^{(\tau)}_k\ge\epsilon$ for all $\tau$ and $k$.
    \end{enumerate}
\end{theorem}
\begin{proof}
    See the Appendix for the proof, which generalizes the proof of convergence of fuzzy K-means in~\cite{groll2005new}.
\end{proof}
It can be easily verified that when the smooth minimum function $h$ is LogSumExp~\eqref{LSE} and p-Norm~\eqref{PNorm} in which case SKM~\eqref{smooth_km_sgd} is equivalent to MEFC and FKM, respectively, the above convergence condition can be satisfied with any initial centroids. Since the Boltzmann operator~\eqref{Boltzmann} is not globally concave, discussing the convergence of EKM is not simple. If centroids are updated within the concave region of the Boltzmann operator, the convergence of EKM can be guaranteed by Theorem~\ref{theorem_convergence}. However, the practical situation may be more complex. A rigorous convergence condition of EKM may involve a tedious discussion of a specific value of $\alpha$, positions of initial centroids, and data structures. Instead, we demonstrate the empirical convergence of EKM in Section~\ref{sect5:convergence}.
\section{Comparison of Different Smoothed Objectives}
\label{sect4}
\subsection{Case Study}
This section presents an empirical analysis of the behavior of different KM algorithms by examining their reformulated objective functions in some examples with well-designed data structures. Datasets comprising two classes of one-dimensional data points are generated by sampling from two normal distributions, drawing $N_1$ samples from a distribution with mean $\mu_1$ and unit variance, and $N_2$ samples from another with mean $\mu_2$ and unit variance. Using different parameter combinations, we generate four datasets: 1. A balanced, non-overlapping dataset ($N_1=N_2=50$, $\mu_1=-5$, and $\mu_2=+5$; Fig.~\ref{fig:balanced_non-overlapping}); 2. A balanced, overlapping dataset ($N_1=N_2=50$, $\mu_1=-0.5$, and $\mu_2=+0.5$; Fig.~\ref{fig:balanced_overlapping}); 3. An imbalanced, non-overlapping dataset ($N_1=5000$, $N_2=50$, $\mu_1=-5$, and $\mu_2=+5$; Fig.~\ref{fig:imbalanced_non-overlapping}); 4. An imbalanced, overlapping dataset ($N_1=2000$, $N_2=50$, $\mu_1=-2$, and $\mu_2=+2$; Fig.~\ref{fig:compact_balanced}).

We plot $J(\mathbf{c}_1=\mu_1, \mathbf{c}_2)$ as a function of $\mathbf{c}_2$ in Fig.~\ref{fig:analysis_objs_1D}, with $\mathbf{c}_2$ being the position of the second centroid on the $x$-axix, and $J$ being the reformulated objectives of HKM (WCSS~\eqref{WCSS}), FKM ($J_{\text{PN}}$~\eqref{PNorm objective}), MEFC ($J_{\text{LSE}}$~\eqref{LSE objective}), and EKM ($J_{\text{B}}$~\eqref{Boltzmann objective}). The four objective functions behave similarly on the first two balanced datasets, but it is worth noting that the last two imbalanced datasets. Fig.~\ref{fig:imbalanced_non-overlapping} and Fig.~\ref{fig:compact_balanced} show that the local and global minimum points of the objective functions of HKM, FKM, and MEFC are biased towards the center of the large cluster, i.e., $\mu_1$. In contrast, EKM does not have an obvious local minimum and its global minimum point aligns with the true cluster center, i.e., $\mu_2$, highlighting EKM's superiority in handling imbalanced data.

\begin{figure}[!t]
\centering
\subfloat[]{\includegraphics[width=0.24\textwidth]{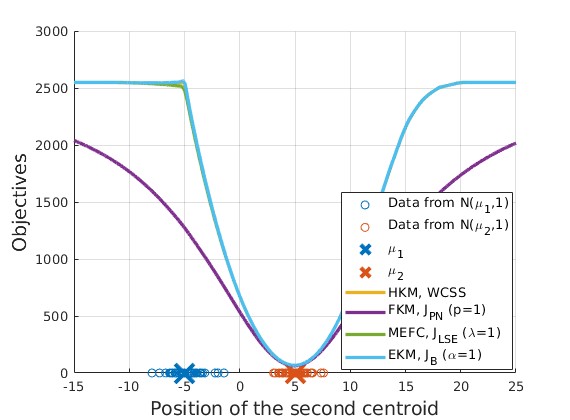}\label{fig:balanced_non-overlapping}}
 \subfloat[]{\includegraphics[width=0.24\textwidth]{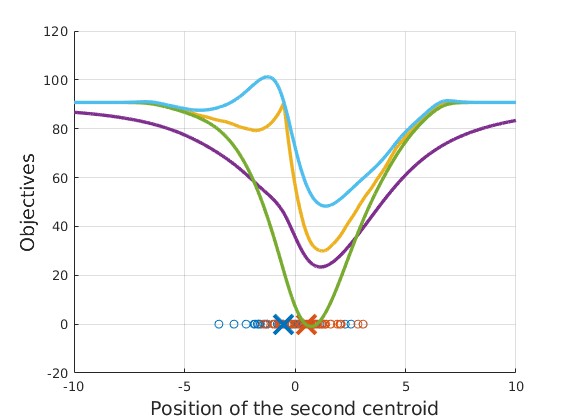}
 \label{fig:balanced_overlapping}}
 \hfill
 \subfloat[]{\includegraphics[width=0.24\textwidth]{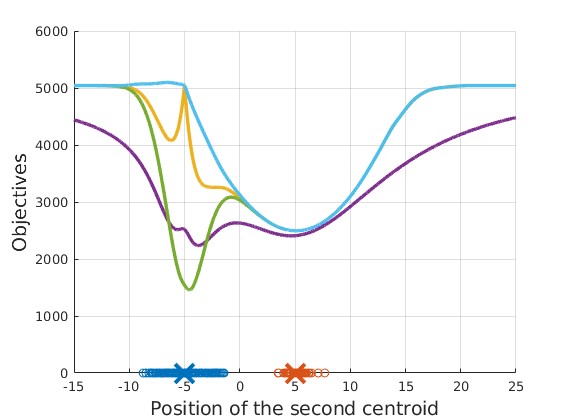}\label{fig:imbalanced_non-overlapping}}
\subfloat[]{\includegraphics[width=0.24\textwidth]{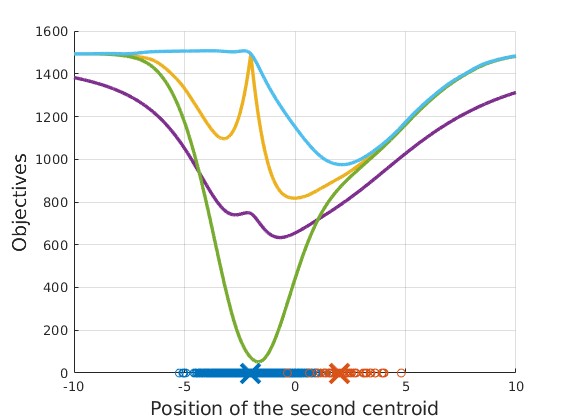}\label{fig:compact_balanced}}
 \caption{Objectives $J({\mathbf{c}_1=\mu_1,\mathbf{c}_2})$ as a function of $\mathbf{c}_2$ (the position of the second centroid on the $x$-axis). The yellow, purple, green, and blue curves are reformulated objective functions of HKM, FKM, MEFC, and EKM, respectively. (a), (b), (c), and (d) present objective functions under different data distributions.}
\label{fig:analysis_objs_1D}
\end{figure}

\subsection{The Analysis of EKM's Effectiveness on Imbalanced Data: A Centroid Repulsion Mechanism}
\label{sect4_2}
We analyze the effectiveness of EKM on imbalanced data based on the gradient of its objective. The gradient of the smoothed WCSS with respect to the data-centroid distance (i.e., $\partial J/\partial d_{kn}$) can be interpreted as the force exerted by a spring. A positive gradient value represents an attractive force, while a negative value represents a repulsive force.

A simple example is given to plot the gradient values of different KM's objectives. We fix two centroids at $-1$ and $+1$, respectively, and move a data point along the $x$-axis. In Fig.~\ref{fig:memerbship_all} we display the gradients of WCSS, $J_\text{LSE}$, $J_\text{PN}$, and $J_\text{B}$ with respect to the distance between the data point and the second centroid. As evident from the figure, data points on the side of the first centroid do not impact the second centroid of HKM, but they do attract the second centroids of FKM and MEFC. This finding supports the claim in~\cite{zhou2020effect} that FKM has a stronger uniform effect than HKM. On the other hand, data points near the first centroid have repulsive forces on the second centroid of EKM, which compensate for the attraction from other data points. This new mechanism effectively reduces the algorithm's learning bias towards other clusters, particularly large ones, because the centers of large clusters often contain more data. The existence of this mechanism suggests that the uniform effect of EKM is weak, a hypothesis that our experiments will test. Notably, data points near the second centroid of EKM exert the strongest attraction forces on this centroid. As a result, EKM's centroids are stabilized by their surrounding data points, reducing their susceptibility to noise and outliers.

\begin{figure}[!t]
    \centering
    \subfloat[]{\includegraphics[width=0.24\textwidth]{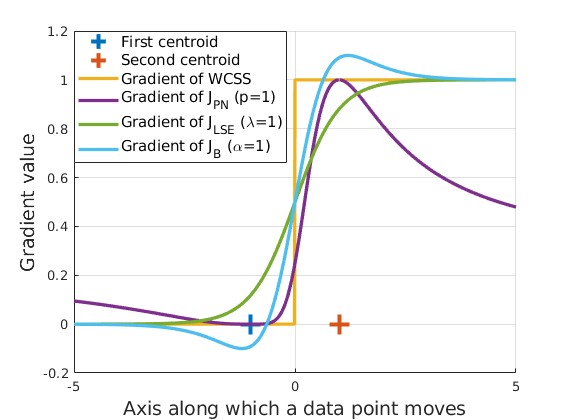} \label{fig:memerbship_all}}
    \subfloat[]{\includegraphics[width=0.24\textwidth]{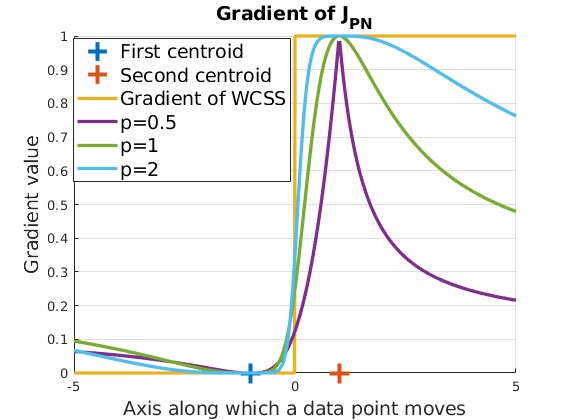} \label{fig:membership_pn}}
    \hfill
    \subfloat[]{\includegraphics[width=0.24\textwidth]{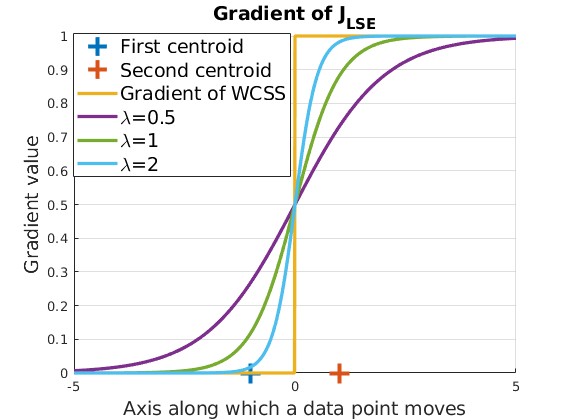} \label{fig:membership_lse}}
    \subfloat[]{\includegraphics[width=0.24\textwidth]{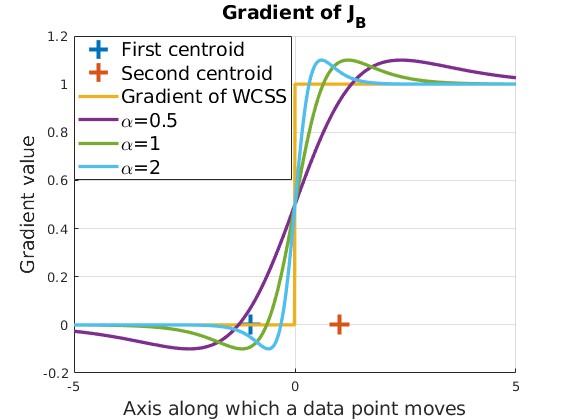} \label{fig:membership_boltzmann}}
    \caption{Gradients as a function of a data point moving along the $x$-axis. (a) Gradients of WCSS, $J_{\text{PN}}$, $J_{\text{LSE}}$, and $J_{\text{B}}$ with fixed smoothing parameters. (b), (c), and (d) are gradients of $J_{\text{PN}}$, $J_{\text{LSE}}$, and $J_{\text{B}}$ with varied smoothing parameters, respectively.}
    \label{fig:membership}
\end{figure}

\subsection{The Choice of EKM's parameter}
\label{sect3_alpha}
The smoothing parameter $\alpha$ impacts the performance of EKM, but the optimal choice remains unknown. This is not a difficulty unique to EKM. The FKM algorithm also struggles with selecting the optimal fuzzifier value, $m$. Despite numerous studies discussing the selection of $m$, a widely accepted solution has yet to be found~\cite{gupta2019fuzzy}.

As a rule of thumb, when the dimension of the data space is less than or equal to three, setting $\alpha=1$ appears effective after normalizing the data to have a zero mean and unit variance for each dimension. As the data space dimension increases, the data-centroid distance $d_{kn}$ increases, necessitating a decrease in $\alpha$ to ensure that the exponential term $e^{-\alpha d_{kn}}$ falls within a normal range. Hence, one can set $\alpha$ as a reciprocal of the data variance or manually tune it as follows. Initially setting $\alpha$ as ten times the reciprocal of data variance and gradually reducing it until a sudden increase in centroid-centroid distance is observed. Because the increase in centroid distance implies the emergence of repulsive forces that reduce the uniform effect.
\section{Numerical Experiments}
\label{sect5}
\subsection{Experimental Setup}
Numerical experiments are conducted to compare the performance of our proposed EKM algorithm with seven related centroid-based algorithms including (1) HKM, (2) FKM, (3) MEFC, (4) PFKM~\cite{pal2005possibilistic}, (5) csiFKM~\cite{noordam2002multivariate}, (6) siibFKM~\cite{lin2014size}, and (7) FWPKM~\cite{yang2020feature}. Multiprototype clustering algorithms (e.g.,~\cite{liang2012k,lu2019self}) are not appropriate as baseline algorithms because they are too complex to be benchmarks for gauging the efficiency of EKM. We also avoid using hybrid methods, such as algorithms combining K-means with kernels proposed in~\cite{zhang2003clustering,huang2011multiple,tang2023knowledge}, as benchmarks to ensure fairness. Implementation is conducted in Matlab R2022a and the operation system is Ubuntu 18.04.1 LTS with Intel Core i9-9900K CPU @ 3.60GHz X 16 and 62.7GiB memory. All datasets used and codes are available at \url{https://github.com/ydcnanhe/Imbalanced-Data-Clustering-using-Equilibrium-K-Means.git}.

The experimental datasets contain four artificial datasets generated by us (including Data-A, Data-B, Data-C, and Data-D), 13 UCI~\cite{asuncion2007uci} datasets (consisting of Image Segmentation (IS), Seeds, Wine, Rice, Wisconsin Diagnostic Breast Cancer (WDBC), Ecoli, Htru2, Zoo, Glass, Shill Bidding, Anuran Calls, Occupancy Detection, Machine Failure), and three Kaggle datasets (incorporating Heart Disease, Pulsar Cleaned, and Bert-Embedded Spam). So there are 20 datasets and Table~\ref{table: dscrp_datasets} provides their information, including name, instance number, feature number, reference class number, and coefficient of variation (CV). CV is used in previous literature~\cite{wu2012advances} to measure the level of data dispersion. It is calculated as the ratio of the standard deviation of class sizes to the mean. Given the number of instances in each class as $N_1,\cdots, N_K$, we have
\begin{equation}
    \text{CV}=s/\Bar{N},
\end{equation}
where
\begin{equation*}
    \bar{N}= \frac{\sum_{k=1}^K N_k}{K},\, s = \sqrt{\frac{\sum_{k=1}^K (N_k-\bar{N})^2}{K-1}}.
\end{equation*}
In~\cite{wu2012advances}, a CV value exceeding 1 indicates highly varying class sizes, while a value below 0.3 signifies uniform class sizes. However, there isn't a widely accepted critical CV value implying that data is imbalanced. For rigorous statements, we establish that if a dataset's CV is less than 0.4, it is considered balanced. Conversely, if the CV exceeds 0.7, the dataset is deemed imbalanced. Consequently, we have six balanced datasets and 14 imbalanced datasets. It should be noted that none of the datasets used has a CV value between 0.4 and 0.7, hence this range remains undefined.

\begin{table}[!t]
\caption{
Detailed description of 20 Datasets
\label{table: dscrp_datasets}}
\centering
{%
    \resizebox{0.45\textwidth}{!}{
        {\begin{tabular}{cccccc}\hline
		ID&Name&Instances&Feature&Classes&CV\\ \hline
            D1&Data-A &2250&2&3&1.4468\\
            D2&Data-B &2250&2&3&1.4468\\
            D3&Data-C &5200&2&2&1.3054\\  
            D4&Data-D &5400&2&9&2.7500\\ \hline \hline
            D5&IS &2310 &19&7 &0\\
            D6&Seeds&210&7&3&0\\
            D7& Heart Disease& 1125& 13& 2&0.0373\\
            D8&Wine &178 &13&3 &0.1939 \\  
            D9&Rice &3810 &7&2 &0.2042 \\
            D10&WDBC &569 &30&3 &0.3604\\ \hline \hline
            D11&Zoo &101 &16&7 &0.8937\\
            D12&Glass &214 &9 &6 &1.0767\\
            D13&Ecoli &336 &7&8 &1.1604\\   
		  D14&Htru2 &17898 &8&2 &1.1552\\
            D15&Shill Bidding& 6321& 9& 2&1.1122\\          
            D16&Anuran Calls& 7195& 22& 10&1.6016\\   
            D17&Occupancy Detection& 20560& 5& 2&0.7608\\  
            D18&Machine Failure & 9815 &7 &2 &1.3318 \\
            D19&Pulsar Cleaned &14987 &7 &2 &1.3561 \\
            D20&Bert-Embedded Spam & 5572 &768 &2 &1.0350 \\ \hline
	\end{tabular}
	}
    }
 }
\end{table}

All datasets undergo normalization, ensuring that each feature has zero mean and unit variance. All features are used for clustering purposes. The number of clusters is set to the reference class number. Convergence is achieved when the moving distance of centroids between successive iterations is sufficiently small relative to the magnitude of centroids, i.e.,
\begin{equation}
\label{convergence_condition}
   \frac{\bigg(\sum_{k=1}^K \|\mathbf{c}_k^{(\tau)}-\mathbf{c}_{k}^{(\tau-1)}\|_2^2\bigg)^{1/2}}{\bigg(\sum_{k=1}^K \|\mathbf{c}_k^{(\tau)}\|_2^2\bigg)^{1/2}}\le 1\mathrm{e}{-3}. 
\end{equation}
We set the maximum number of iterations to 500 to prevent algorithms from infinitely iterating due to failure to converge. Except for sporadic cases, all algorithms can reach the specified convergence within 500 iterations.

Centroids are initialized using the K-means++ algorithm~\cite{arthur2007k}. Given that each run of K-means++ yields different outputs and the objectives of the tested clustering algorithms are known to be non-convex (i.e., multiple local optima exist), convergence may happen at one of the local optimal points. A common way of finding the global optimum is to carry out a number of replications followed by a selection of the best (lowest) objective value. Hence, each trial includes 100 repetitions, and we select the repetition with the lowest objective value as the final result for that trial. We average the performance of 50 trials to ensure the rationality of the experiment. This means we conduct $50\times100=5000$ runs for each algorithm. 100 repetitions in each trial is still within the practical executable range and each trial can give almost consistent results as indicated by a small standard deviation. We utilize widely accepted measures of clustering performance for our evaluation indexes, including the normalized mutual information (NMI)~\cite{strehl2002cluster}, the adjusted rand index (ARI)~\cite{yeung2001details}, and the clustering accuracy index (ACC)~\cite{graves2010kernel}. Both NMI and ACC range from 0 to 1, with 0 indicating the worst, and 1 representing the best. ARI has a range from -1 to 1, in which -1 means two data clusterings are completely dissimilar, 0 indicates that data clusterings are essentially random, and 1 means two data clusterings are perfectly aligned.

FKM, PFKM, csiFKM, siibFKM, and FWPKM employ a typical fuzzifier value of $m=2$~\cite{huang2012range}. The fuzzifier value of MEFC is set to $\lambda=1$. For EKM, we set $\alpha=1$ for the four artificial datasets. For the 16 real-world datasets with varying feature numbers, the parameter $\alpha$ is set proportional to the data variance, as follows
\begin{equation}
    \alpha= 2/\bar{d}_{0n},
\end{equation}
where $\bar{d}_{0n}=\frac{1}{2}\sum_{n=1}^N \|\mathbf{x}_n\|_2^2/N$, $\mathbf{x}_n$ is the $n$-th data point, and $N$ is the total number of data points. The reason for such selection is explained in Section~\ref{sect3_alpha}.

\subsection{Artificial Datasets}
Fig.~\ref{fig:artificial_datasets} displays scatter plots of the four artificial datasets along with selected data clusterings. The average and the standard deviation of evaluation indexes over 50 trials are provided in Table~\ref{table::artificial_datasets}, where the best and the second-best results are highlighted in bold for easy comparison. On Data-A, the data points are sampled from Gaussian distributions. While on Data-B the data points are sampled from uniform distributions. Data-C and Data-D are a mixture of Gaussian and uniform distributions. The difference is that the majority of data points on Data-C are sampled from a Gaussian distribution while the majority of data points on Data-D are sampled from a uniform distribution. Additionally, Data-D has more classes than Data-C. All four artificial datasets are highly imbalanced with CV values greater than one.

We can see from Table~\ref{table::artificial_datasets} that the performance of benchmark algorithms is poor on these artificial datasets, especially on Data-C and Data-D. Nevertheless, the proposed EKM is effective, noticeably outperforming the best benchmark algorithm. The csiFKM has a certain effectiveness on Data-A and Data-B but does not exceed EKM. As seen in Fig.~\ref{fig:artificial_datasets}, the benchmark algorithms erroneously divide the majority class into multiple clusters to balance the data size (i.e., the uniform effect) while EKM does not. This experiment also proves that EKM is versatile in handling different data distributions.

\begin{figure*}[!t]
\centering
    \subfloat[]{\includegraphics[width=0.23\textwidth]{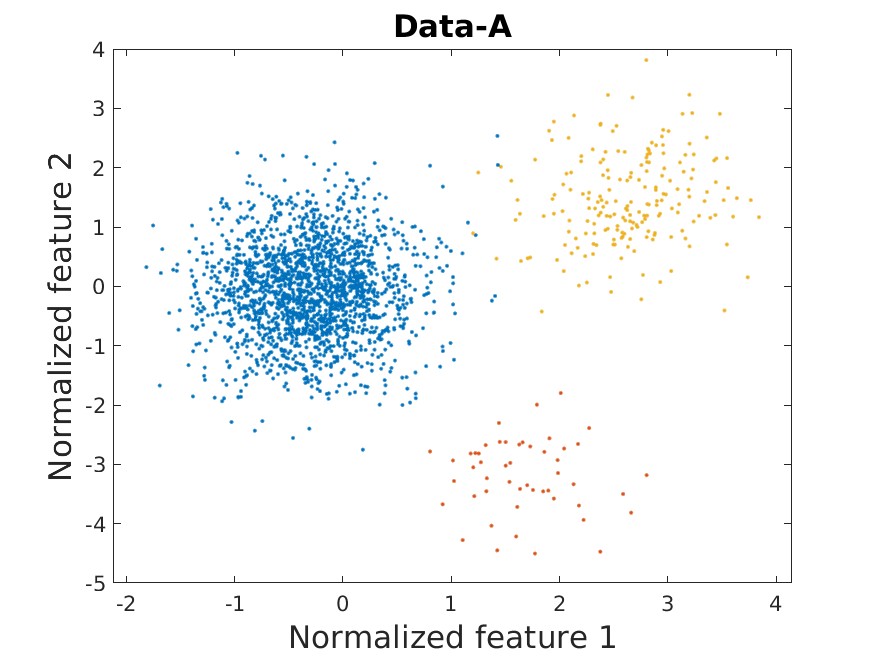}}
    \subfloat[]{\includegraphics[width=0.23\textwidth]{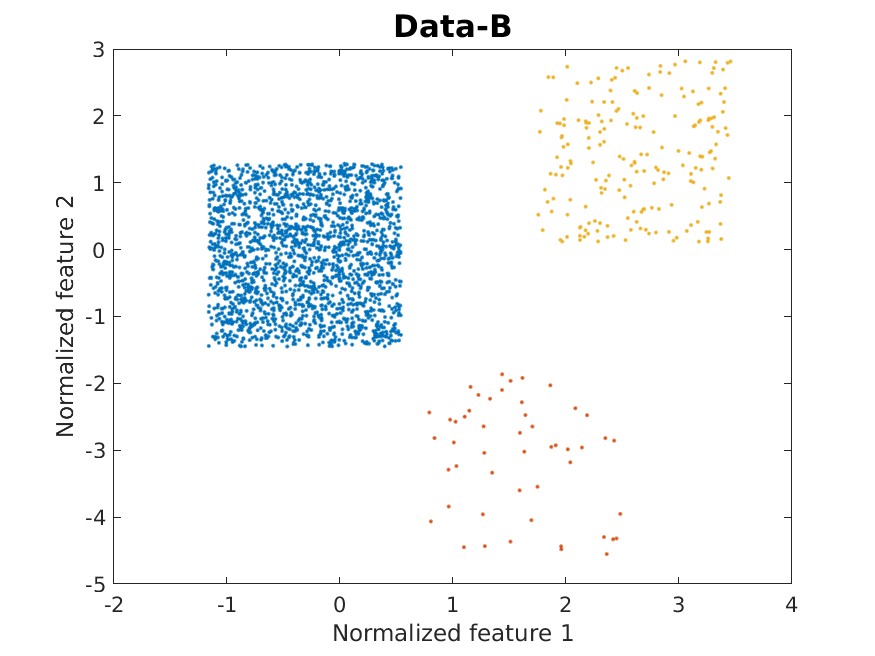}}
    \subfloat[]{\includegraphics[width=0.23\textwidth]{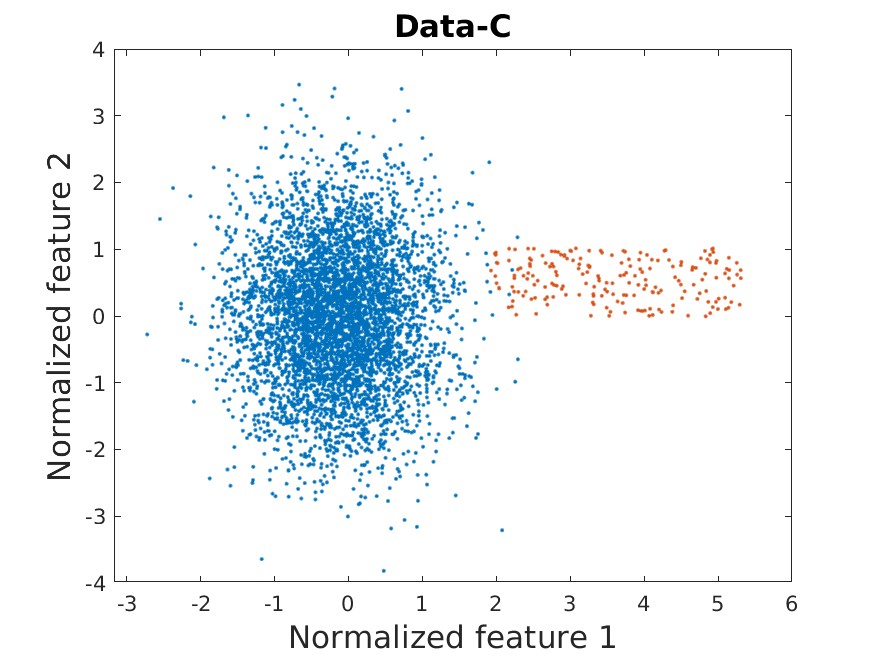}}
    \subfloat[]{\includegraphics[width=0.23\textwidth]{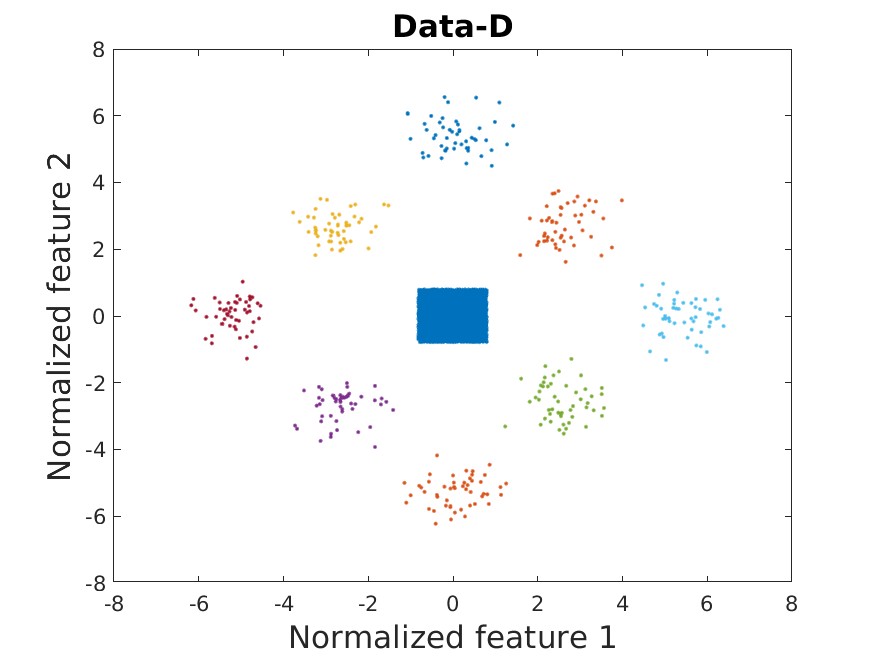}}\\
    \subfloat[]{\includegraphics[width=0.23\textwidth]{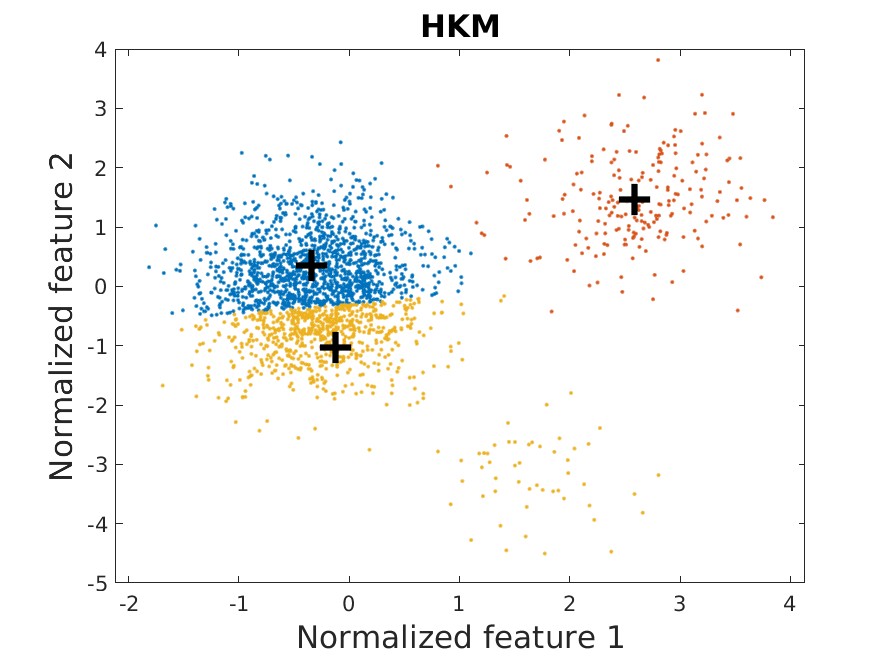}}
    \subfloat[]{\includegraphics[width=0.23\textwidth]{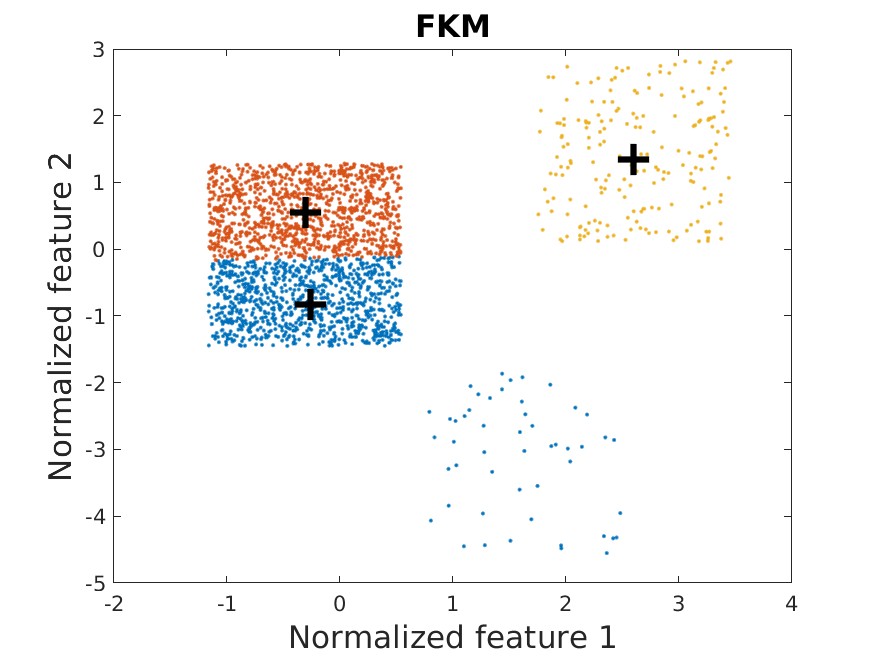}}
    \subfloat[]{\includegraphics[width=0.23\textwidth]{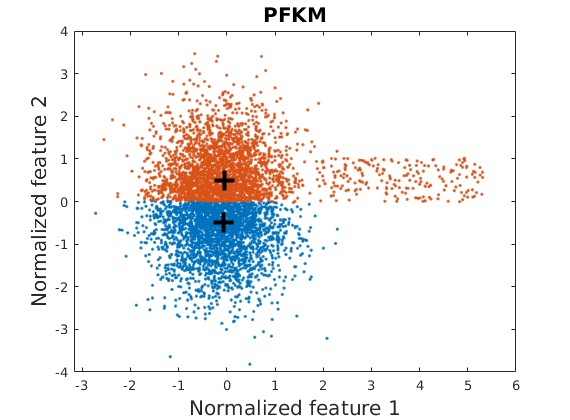}}
    \subfloat[]{\includegraphics[width=0.23\textwidth]{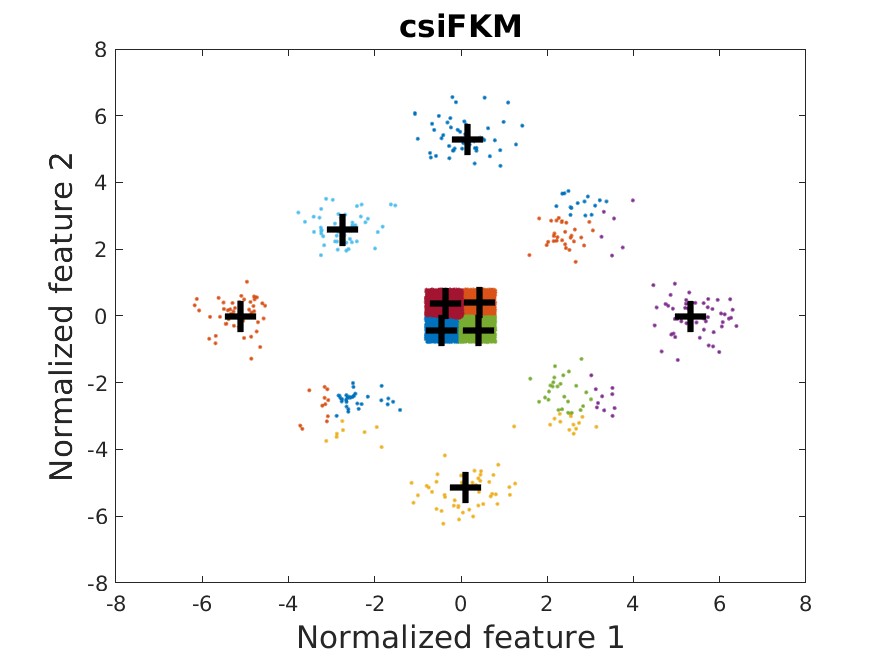}}\\
    \subfloat[]{\includegraphics[width=0.23\textwidth]{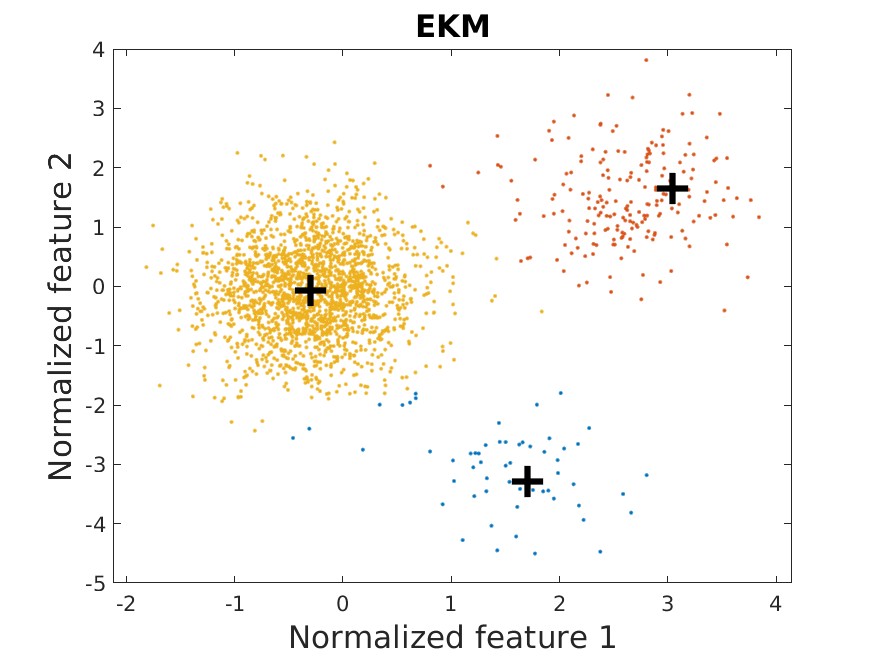}}
    \subfloat[]{\includegraphics[width=0.23\textwidth]{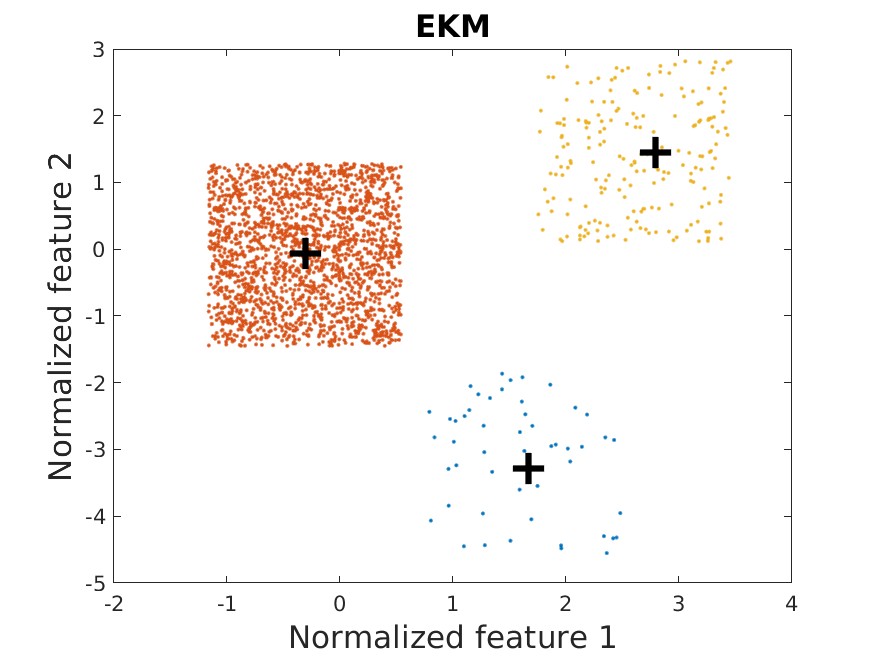}}
    \subfloat[]{\includegraphics[width=0.23\textwidth]{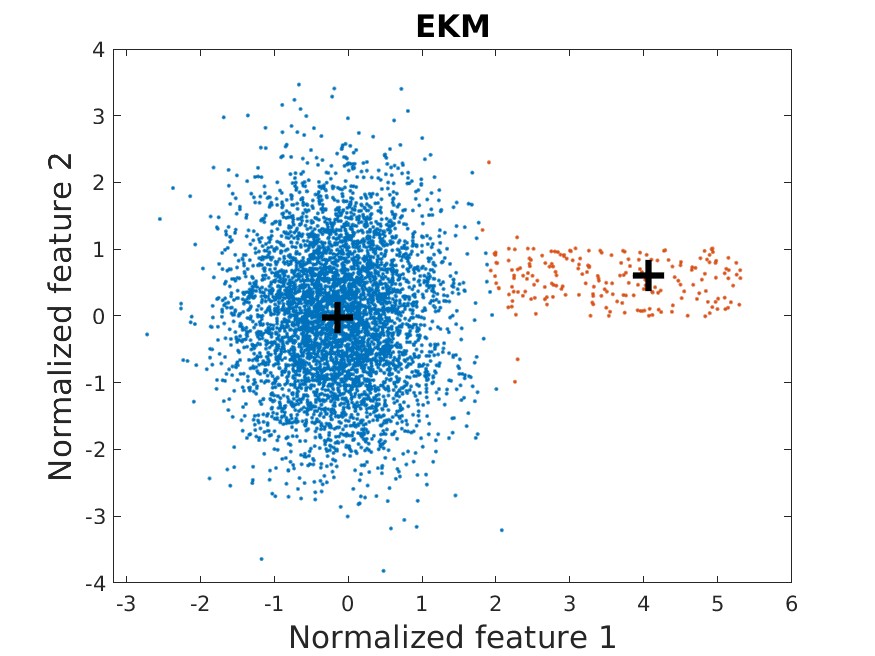}}
    \subfloat[]{\includegraphics[width=0.23\textwidth]{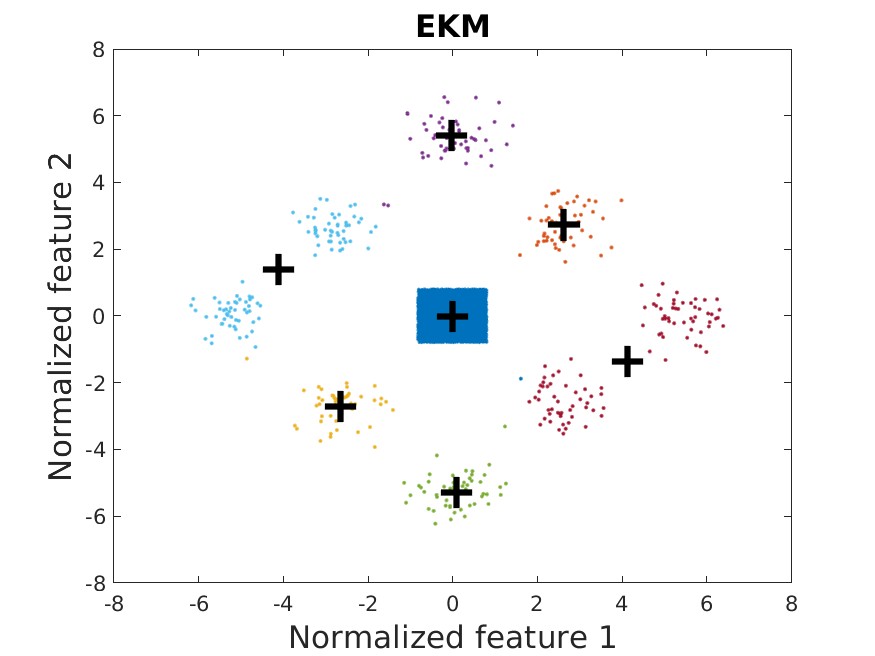}\label{fig:data-D_ekm}}

 \caption{Scatter diagrams of artificial datasets. (a)-(d) Primitive scatter diagrams with reference class labels (indicated by colors). (e)-(f) Clustering results of some benchmark algorithms. (i)-(l) Clustering results of the proposed EKM. The black crosses represent centroids obtained by each algorithm.}
\label{fig:artificial_datasets}
\end{figure*}

\begin{table}[!t]
\caption{
Experimental results on four artificial datasets with $CV>1$ (the best and second-best performances are in bold)
\label{table::artificial_datasets}}
\centering
{%
    \resizebox{0.48\textwidth}{!}{
        {\begin{tabular}{c|c|c|c|c|c|c|c|c|c} \hline
		Dataset&Measurement&HKM &FKM &MEFC  &PFKM &csiFKM  &siibFKM &FWPKM &EKM \\ \hline
            \multirow{3}{*}{Data-A}
            
                                  &NMI &\makecell{0.5150\\$\pm$0.0001} &\makecell{0.4982\\$\pm$0.0001} &\makecell{0.2270\\$\pm$0.0002}  &\makecell{0.4984\\$\pm$0.0001}
                                       &\makecell{\textbf{0.7579}\\$\pm$\textbf{0.0471}} &\makecell{0.4924\\$\pm$0.0503}  &\makecell{0.1427\\$\pm$0.0488}  &\makecell{\textbf{0.9126}\\$\pm$\textbf{0.0000}}\\ \cline{2-10}
                                  &ARI &\makecell{0.2906\\$\pm$0.0002} &\makecell{0.2485\\$\pm$0.0002} &\makecell{0.0756\\$\pm$0.0003}  &\makecell{0.2489\\$\pm$0.0003} 
                                       &\makecell{\textbf{0.8545}\\$\pm$\textbf{0.0329}} &\makecell{0.2974\\$\pm$0.0490}  &\makecell{0.0116\\$\pm$0.0357}  &\makecell{\textbf{0.9616}\\$\pm$\textbf{0.0000}}\\ \cline{2-10}
                                  &ACC &\makecell{0.6893\\$\pm$0.0003} &\makecell{0.6075\\$\pm$0.0006} &\makecell{0.4818\\$\pm$0.0003}  &\makecell{0.6089\\$\pm$0.0009}
                                       &\makecell{\textbf{0.9429}\\$\pm$\textbf{0.0164}} &\makecell{0.7443\\$\pm$0.0192} &\makecell{0.4562\\$\pm$0.0393}  &\makecell{\textbf{0.9933}\\$\pm$\textbf{0.0000}}\\  \cline{2-10}
                                       \hline
                                       
            \multirow{3}{*}{Data-B}
            
                                  &NMI &\makecell{0.5193\\$\pm$0.0000} &\makecell{0.5160\\$\pm$0.0000} &\makecell{0.2121\\$\pm$0.0001}  &\makecell{0.5181\\$\pm$0.0001}
                                       &\makecell{\textbf{0.7989}\\$\pm$\textbf{0.0677}} &\makecell{0.5090\\$\pm$0.0419}  &\makecell{0.1591\\$\pm$0.0577}  &\makecell{\textbf{0.9981}\\$\pm$\textbf{0.0057}}\\ \cline{2-10}
                                  &ARI &\makecell{0.2529\\$\pm$0.0000} &\makecell{0.2452\\$\pm$0.0000} &\makecell{0.0547\\$\pm$0.0000}  &\makecell{0.2498\\$\pm$0.0003}
                                       &\makecell{\textbf{0.8794}\\$\pm$\textbf{0.0406}} &\makecell{0.2461\\$\pm$0.0486}  &\makecell{0.0332\\$\pm$0.0497}  &\makecell{\textbf{0.9992}\\$\pm$\textbf{0.0023}}\\ \cline{2-10}
                                  &ACC &\makecell{0.6116\\$\pm$0.0000} &\makecell{0.5852\\$\pm$0.0002} &\makecell{0.4607\\$\pm$0.0016}  &\makecell{0.6022\\$\pm$0.0009}
                                       &\makecell{\textbf{0.9424}\\$\pm$\textbf{0.0194}} &\makecell{0.6304\\$\pm$0.0278}  &\makecell{0.4767\\$\pm$0.0671}  &\makecell{\textbf{0.9999}\\$\pm$\textbf{0.0004}}\\ \cline{2-10}
                                       \hline
                                  
          \multirow{3}{*}{Data-C}
          
                                  &NMI &\makecell{0.0909\\$\pm$0.0000} &\makecell{0.0797\\$\pm$0.0000} &\makecell{0.0892\\$\pm$0.0000}  &\makecell{0.0782\\$\pm$0.0001}
                                       &\makecell{\textbf{0.1069}\\$\pm$\textbf{0.0928}} &\makecell{0.0769\\$\pm$0.1281}  &\makecell{0.0783\\$\pm$0.0151}  &\makecell{\textbf{0.9514}\\$\pm$\textbf{0.0000}}\\ \cline{2-10}
                                      
		                       &ARI &\makecell{0.0181\\$\pm$0.0000} &\makecell{0.0034\\$\pm$0.0001} &\makecell{0.0159\\$\pm$0.0001}  &\makecell{0.0015\\$\pm$0.0001}
                                       &\makecell{\textbf{0.0728}\\$\pm$\textbf{0.1242}} &\makecell{0.0199\\$\pm$0.1713}  &\makecell{0.0053\\$\pm$0.0090}  &\makecell{\textbf{0.9807}\\$\pm$\textbf{0.0000}}\\   \cline{2-10}
                                       
                                  &ACC &\makecell{0.5754\\$\pm$0.0000} &\makecell{0.5302\\$\pm$0.0002} &\makecell{0.5689\\$\pm$0.0002}  &\makecell{0.5237\\$\pm$0.0002}
                                       &\makecell{0.6981\\$\pm$0.2108} &\makecell{\textbf{0.8207}\\$\pm$\textbf{0.0436}}  &\makecell{0.5371\\$\pm$0.0272}  &\makecell{\textbf{0.9987}\\$\pm$\textbf{0.0000}}\\  \hline
                                       
            \multirow{3}{*}{Data-D}
            
                                  &NMI &\makecell{\textbf{0.4799}\\$\pm$\textbf{0.0316}} &\makecell{0.3185\\$\pm$0.0010} &\makecell{0.3963\\$\pm$0.0670}  &\makecell{0.2353\\$\pm$0.0554}
                                       &\makecell{0.3984\\$\pm$0.0440} &\makecell{0.3022\\$\pm$0.0729}  &\makecell{0.1733\\$\pm$0.0268}  &\makecell{\textbf{0.9463}\\$\pm$\textbf{0.0104}}\\ \cline{2-10}
                                       
		                       &ARI &\makecell{\textbf{0.1096}\\$\pm$\textbf{0.0195}} &\makecell{0.0467\\$\pm$0.0007} &\makecell{0.1506\\$\pm$0.0360}  &\makecell{0.0389\\$\pm$0.0183}
                                       &\makecell{0.0813\\$\pm$0.0266} &\makecell{0.0339\\$\pm$0.0260}  &\makecell{0.0083\\$\pm$0.0068}  &\makecell{\textbf{0.9954}\\$\pm$\textbf{0.0049}}\\  \cline{2-10}

                                  &ACC &\makecell{0.3593\\$\pm$0.0555} &\makecell{0.2559\\$\pm$0.0080} &\makecell{\textbf{0.4893}\\$\pm$\textbf{0.0133}}  &\makecell{0.3699\\$\pm$0.0380}
                                       &\makecell{0.3357\\$\pm$0.0564 } &\makecell{0.2330\\$\pm$0.0398}  &\makecell{0.2206\\$\pm$0.0325}  &\makecell{\textbf{0.9763}\\$\pm$\textbf{0.0045}}\\ \hline

	\end{tabular}
    }
	}
 }
\end{table}

\subsection{Study of Parameter Impact}
\label{sect5_param_impact}
We investigate the influence of the parameter $\alpha$ on EKM in this experiment. We perform EKM with $\alpha$ values of $0.1,0.2,0.5,0.8,1,2, 5, 8, 10$ on the four artificial datasets. Fig.~\ref{fig:nmi_alpha} presents the corresponding NMI values for these $\alpha$ values. The NMI value of HKM is considered as a reference. It is evident that when $\alpha>5$, the NMI of EKM closely aligns with that of HKM. This similarity arises because a sufficiently large $\alpha$ makes the objective of EKM nearly identical to the objective of HKM. As the value of $\alpha$ decreases, the NMI of EKM increases abruptly, which implies the emergence of repulsive force between centroids overcoming the uniform effect. When $\alpha<0.2$, EKM becomes inferior to HKM. This is because the centroids of EKM tend to overlap when $\alpha$ is particularly small (repulsion becomes attraction, see Fig.~\ref{fig:membership_boltzmann} when $\alpha=0.5$). Fortunately, the range of $\alpha$ that causes the centroids to overlap is narrow.
\begin{figure}[!t]
\centering
\subfloat[]{\includegraphics[width=0.24\textwidth]{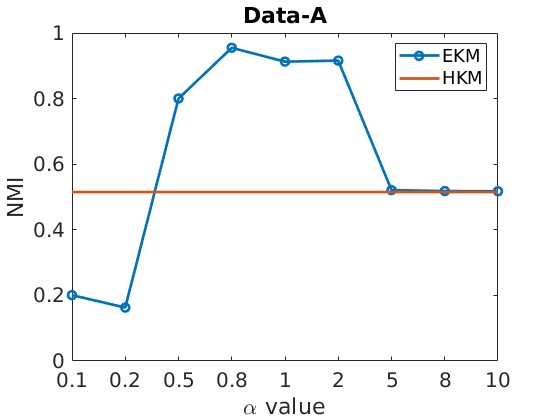}}
 \subfloat[]{\includegraphics[width=0.24\textwidth]{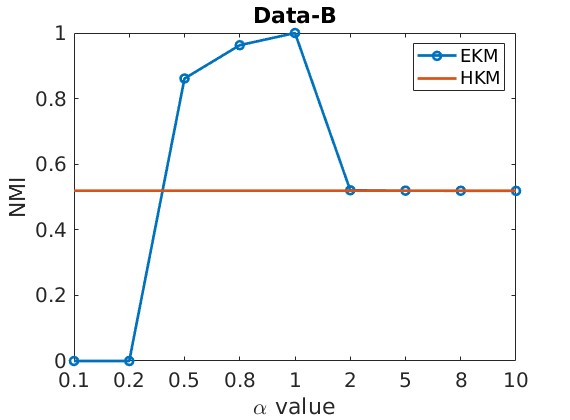}}
 \hfill
 \subfloat[]{\includegraphics[width=0.24\textwidth]{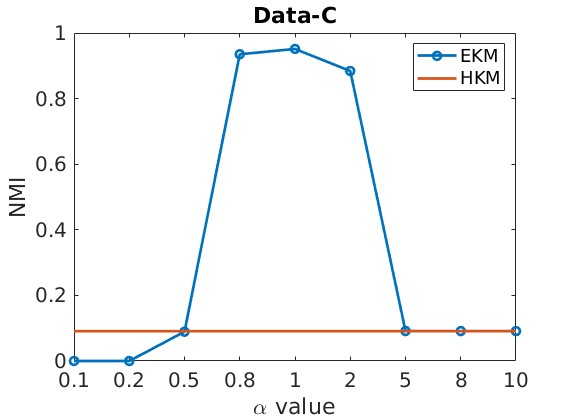}}
\subfloat[]{\includegraphics[width=0.24\textwidth]{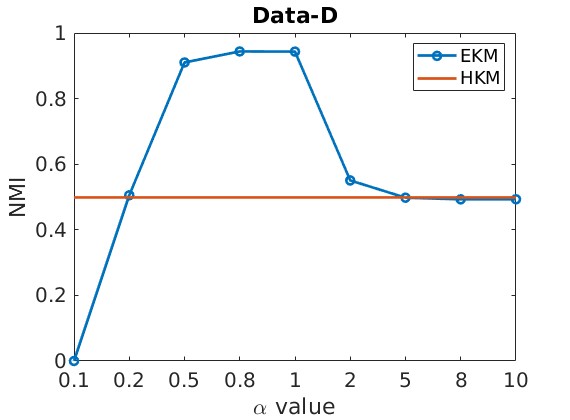}}
 \caption{NMI of EKM with different $\alpha$ values on the four artificial datasets. The blue curve is the NMI of EKM, and the red line is the NMI of HKM as a reference.}
\label{fig:nmi_alpha}
\end{figure}

\subsection{Real Datasets with Balanced Data}
 To evaluate EKM's effectiveness on balanced data, we perform it on six real datasets: IS, Seeds, Heart Disease, Wine, Rice, and WDBC. These datasets are routinely used in the field of machine learning and are characterized by uniform class sizes, CV values less than 0.4, and are categorized as balanced datasets as per our protocol. Table~\ref{table:balanced_performance} shows the clustering results. EKM delivers superior performance on the Wine dataset and is comparable to the baseline performance on the remaining five datasets. Although EKM's mechanism makes it more advantageous for imbalanced data, this experiment proves that EKM is competitive for balanced data.

\begin{table}[!t]
\caption{
Experimental results on selected six real-world datasets with $CV<0.4$ (The best performance is in bold)
\label{table:balanced_performance}}
\centering
{%
    \resizebox{0.48\textwidth}{!}{
        {\begin{tabular}{c|c|c|c|c|c|c|c|c|c} \hline
		Dataset&Measurement&HKM &FKM &MEFC  &PFKM&csiFKM &siibFKM &FWPKM&EKM  \\ \hline
            \multirow{3}{*}{IS}   
                                  &NMI &\makecell{0.5873\\$\pm$0.0008} &\makecell{0.5950\\$\pm$0.0057} &\makecell{0.6206\\$\pm$0.0009}  &\makecell{0.4962\\$\pm$0.0241}
                                       &\makecell{0.5195\\$\pm$0.0317} &\makecell{0.5925\\$\pm$0.0371}  &\makecell{0.5452\\$\pm$0.0664}  &\makecell{\textbf{0.6618}\\$\pm$\textbf{0.0138}}\\ \cline{2-10}
		                       &ARI &\makecell{0.4608\\$\pm$0.0005} &\makecell{0.4960\\$\pm$0.0059} &\makecell{\textbf{0.4979}\\$\pm$\textbf{0.0005}}  &\makecell{0.3650\\$\pm$0.0149}
                                       &\makecell{0.3472\\$\pm$0.0587} &\makecell{0.3776\\$\pm$0.0794}  &\makecell{0.4022\\$\pm$0.0884}  &\makecell{0.4810\\$\pm$0.0089}\\   \cline{2-10}
                                  &ACC &\makecell{0.5456\\$\pm$0.0003} &\makecell{\textbf{0.6578}\\$\pm$\textbf{0.0146}} &\makecell{0.5943\\$\pm$0.0002}  &\makecell{0.5276\\$\pm$0.0139}
                                       &\makecell{0.4977\\$\pm$0.0670} &\makecell{0.5840\\$\pm$0.0511}  &\makecell{0.5576\\$\pm$0.0775}  &\makecell{0.5609\\$\pm$0.0118}\\  
                                  \hline
            \multirow{3}{*}{Seeds}   
                                  &NMI &\makecell{0.7279\\$\pm$0.0000} &\makecell{0.7318\\$\pm$0.0056} &\makecell{\textbf{0.7496}\\$\pm$\textbf{0.0016}}  &\makecell{0.7170\\$\pm$0.0000}
                                       &\makecell{0.7200\\$\pm$0.0077} &\makecell{0.6893\\$\pm$0.0015}  &\makecell{0.5409\\$\pm$0.0336}  &\makecell{0.7315\\$\pm$0.0000} \\  \cline{2-10}
		                       &ARI &\makecell{0.7733\\$\pm$0.0000} &\makecell{0.7768\\$\pm$0.0058} &\makecell{\textbf{0.7968}\\$\pm$\textbf{0.0017}}  &\makecell{0.7607\\$\pm$0.0000}
                                       &\makecell{0.7640\\$\pm$0.0082} &\makecell{0.7143\\$\pm$0.0034}  &\makecell{0.4724\\$\pm$0.0513}  &\makecell{0.7715\\$\pm$0.0000} \\    \cline{2-10}
                                  &ACC &\makecell{0.9190\\$\pm$0.0000} &\makecell{0.9209\\$\pm$0.0023} &\makecell{\textbf{0.9285}\\$\pm$\textbf{0.0007}}  &\makecell{0.9143\\$\pm$0.0000}
                                       &\makecell{0.9156\\$\pm$0.0033} &\makecell{0.8952\\$\pm$0.0015}  &\makecell{0.7069\\$\pm$0.0649}  &\makecell{0.9190\\$\pm$0.0000} \\ 
                                  \hline
            \multirow{3}{*}{Heart Disease} 
                                  &NMI
                                       &\makecell{\textbf{0.3162}\\$\pm$\textbf{0.0000}} &\makecell{0.0142\\$\pm$0.0132} &\makecell{0.2500\\$\pm$0.0000}  &\makecell{0.2326\\$\pm$0.0975}
                                       &\makecell{0.1548\\$\pm$0.0398} &\makecell{0.0364\\$\pm$0.0476}  &\makecell{0.0868\\$\pm$0.0820}  &\makecell{0.2571\\$\pm$ 0.0000}\\ \cline{2-10}
		                       &ARI 
                                       &\makecell{\textbf{0.3641}\\$\pm$\textbf{0.0000}} &\makecell{0.0033\\$\pm$0.0042} &\makecell{0.3140\\$\pm$0.0000}  &\makecell{0.2995\\$\pm$0.1230}
                                       &\makecell{0.0509\\$\pm$0.0172} &\makecell{0.0163\\$\pm$0.0523}  &\makecell{0.1030\\$\pm$0.1050}  &\makecell{0.3162\\$\pm$ 0.0000}\\  \cline{2-10}
                                  &ACC 
                                       &\makecell{\textbf{0.8020}\\$\pm$\textbf{0.0000}} &\makecell{0.5249\\$\pm$0.0202} &\makecell{0.7805\\$\pm$0.0000}  &\makecell{0.7653\\$\pm$0.0689}
                                       &\makecell{0.6101\\$\pm$0.0284} &\makecell{0.5354\\$\pm$0.0549}  &\makecell{0.6344\\$\pm$0.0898}  &\makecell{0.7815\\$\pm$0.0000}\\ 
                                  \hline 
                                  
            \multirow{3}{*}{Wine} 
                                  &NMI &\makecell{0.8759\\$\pm$0.0000} &\makecell{0.8759\\$\pm$0.0000} &\makecell{0.8759\\$\pm$0.0000}  &\makecell{0.8097\\$\pm$0.0171}
                                       &\makecell{0.6966\\$\pm$0.1663} &\makecell{0.5250\\$\pm$0.0803}  &\makecell{0.3976\\$\pm$0.0794}  &\makecell{\textbf{0.8920}\\$\pm$\textbf{0.0000}}\\ \cline{2-10}
		                       &ARI &\makecell{0.8975\\$\pm$0.0000} &\makecell{0.8975\\$\pm$0.0000} &\makecell{0.8975\\$\pm$0.0000}  &\makecell{0.8249\\$\pm$0.0142}
                                       &\makecell{0.6283\\$\pm$0.2549} &\makecell{0.3954\\$\pm$0.1252}  &\makecell{0.3518\\$\pm$0.0987}  &\makecell{\textbf{0.9134}\\$\pm$\textbf{0.0000}}\\  \cline{2-10}
                                  &ACC &\makecell{0.9663\\$\pm$0.0000} &\makecell{0.9663\\$\pm$0.0000} &\makecell{0.9663\\$\pm$0.0000}  &\makecell{0.9396\\$\pm$0.0051}
                                       &\makecell{0.7797\\$\pm$0.1834} &\makecell{0.7129\\$\pm$0.0900}  &\makecell{0.6476\\$\pm$0.0742}  &\makecell{\textbf{0.9719}\\$\pm$\textbf{0.0000}}\\ 
                                  \hline
            \multirow{3}{*}{Rice}
            
                                  &NMI &\makecell{0.5685\\$\pm$0.0000} &\makecell{0.5688\\$\pm$0.0000} &\makecell{0.5682\\$\pm$0.0004}  &\makecell{\textbf{0.5742}\\$\pm$\textbf{0.0008}}
                                       &\makecell{0.5544\\$\pm$0.0001} &\makecell{0.5437\\$\pm$0.0032}  &\makecell{0.0536\\$\pm$0.0268}  &\makecell{0.5659\\$\pm$0.0005}\\ \cline{2-10}
		                       &ARI &\makecell{0.6815\\$\pm$0.0000} &\makecell{0.6824\\$\pm$0.0000} &\makecell{0.6817\\$\pm$0.0004}  &\makecell{\textbf{0.6876}\\$\pm$\textbf{0.0008}}
                                       &\makecell{0.6634\\$\pm$0.0001} &\makecell{0.6471\\$\pm$0.0109}  &\makecell{0.0338\\$\pm$0.0134}  &\makecell{0.6777\\$\pm$0.0004}\\ \cline{2-10}
                                  &ACC &\makecell{0.9129\\$\pm$0.0000} &\makecell{0.9131\\$\pm$0.0000} &\makecell{0.9129\\$\pm$0.0001}  &\makecell{\textbf{0.9147}\\$\pm$\textbf{0.0002}}
                                       &\makecell{0.9074\\$\pm$0.0000} &\makecell{0.9023\\$\pm$0.0033}  &\makecell{0.5941\\$\pm$0.0155}  &\makecell{0.9117\\$\pm$0.0001}\\ 
                                  \hline
                                  
            \multirow{3}{*}{WBDC} 
                                
                                  &NMI &\makecell{0.5547\\$\pm$0.0000} &\makecell{0.5612\\$\pm$0.0000} &\makecell{0.5547\\$\pm$0.0000}  &\makecell{\textbf{0.5700}\\$\pm$\textbf{0.0000}}
                                       &\makecell{0.3516\\$\pm$0.2432} &\makecell{0.4202\\$\pm$0.1246 }  &\makecell{0.1138\\$\pm$0.0626}  &\makecell{0.5513\\$\pm$0.0025}\\ \cline{2-10}
		                       &ARI &\makecell{0.6707\\$\pm$0.0000} &\makecell{0.6829\\$\pm$0.0000} &\makecell{0.6707\\$\pm$0.0000}  &\makecell{\textbf{0.6895}\\$\pm$\textbf{0.0000}}
                                       &\makecell{0.4093\\$\pm$0.3070} &\makecell{0.5250\\$\pm$0.1734}  &\makecell{0.0281\\$\pm$0.0409}  &\makecell{0.6444\\$\pm$0.0028}\\   \cline{2-10}
                                  &ACC &\makecell{0.9104\\$\pm$0.0000} &\makecell{0.9139\\$\pm$0.0000} &\makecell{0.9104\\$\pm$0.0000}  &\makecell{\textbf{0.9156}\\$\pm$\textbf{0.0000}}
                                       &\makecell{0.7989\\$\pm$0.1380} &\makecell{0.8557\\$\pm$0.0810}  &\makecell{0.5910\\$\pm$0.0441}  &\makecell{0.9027\\$\pm$0.0009}\\ 
                                  \hline
            
	\end{tabular}
    }
	}
 }
\end{table}
\subsection{Real Datasets with Imbalanced Data}
The considered datasets here are Ecoli, Htru2, Zoo, Glass, Shill Bidding, Anuran Calls, Occupancy Detection, Machine Failure, Pulsar Cleaned, and Bert-Embedded Spam. Among them, Htru2 and Occupancy Detection have a large amount of data, Glass and Ecoli have 6 and 8 classes respectively, and Zoo has high dimensionality. They have varying class sizes with CV values greater than 0.7. Except for Occupancy Detection and Zoo, the other eight datasets have CV values greater than 1. Bert-Embedded Spam has over 700 features. Since the Euclidean distance is ineffective for clustering algorithms on high-dimensional data, we execute principal component analysis on the Bert-Embedded Spam data. The first five principal components are used for clustering as adding more does not improve and even compromise performance. Table~\ref{table:imbalanced_performance} displays the results, where the best and the second-best results are in bold.

As can be seen from Table~\ref{table:imbalanced_performance}, in general, EKM obtains the best performances on eight datasets, and achieves the second-best clustering performances on the remaining two datasets (Zoo and Anuran Calls), as confirmed by all three measured clustering indexes. Specifically, although csiFKM outperforms EKM on the Zoo dataset, it performs substantially worse on Htru2, Occupancy Detection, Shill Bidding, etc. HKM performs well on Anuran Calls but underperforms on Shill Bidding, Machine Failure, and Bert-Embedded Spam. Due to the structural complexity of real datasets, it is reasonable that some algorithms perform well on certain datasets. In comparison, EKM performs consistently well on all datasets, especially on the Htru2, Shill Bidding, Machine Failure, and Bert-Embedded Spam datasets, where it significantly outperforms the benchmark algorithms. Fig.~\ref{fig:imbalanced_datasets} illustrates the scatter diagrams of Htru2, Shill Bidding, Machine Failure, and Bert-Embedded Spam datasets, along with some clustering results. Comparing Fig.~\subref*{fig:htru2_fwpkm} - \subref*{fig:spam_mefc} with Fig.~\subref*{fig:htru2_ekm} - \subref*{fig:spam_ekm}, again, we can see the uniform effect exists in the benchmark algorithms. Conversely, the centroids obtained by EKM are farther away from each other, implying the existence of the repulsive force overcoming the uniform effect. Note that, we can only see one centroid in some figures because centroids overlap on the selected two coordinates. Overall, the experiment results provide strong evidence to support the superiority and consistent performance of EKM on imbalanced data. Note that, we simply set $\alpha$ proportional to the data variance. EKM's performance can be further enhanced by fine-tuning the parameter $\alpha$.

\begin{table*}[!t]
\caption{
Experimental results on selected ten real-world datasets with $CV>0.7$ (the best and second-best performances are in bold)
\label{table:imbalanced_performance}}
\centering
{%
    \resizebox{0.95\textwidth}{!}{
        {\begin{tabular}{c|c|c|c|c|c|c|c|c|c} \hline
		Dataset&Measurement&HKM &FKM &MEFC  &PFKM&csiFKM &siibFKM &FWPKM&EKM   \\ \hline
            \multirow{3}{*}{Zoo} 
            
                                  &NMI &\makecell{\textbf{0.8381}\\$\pm$\textbf{0.0268}} &\makecell{0.7764\\$\pm$0.0019} &\makecell{0.8179\\$\pm$0.0000}  &\makecell{0.6611\\$\pm$0.0806}
                                       &\makecell{\textbf{0.8657}\\$\pm$\textbf{0.0054}} &\makecell{0.7625\\$\pm$0.0415}  &\makecell{0.6652\\$\pm$0.0804}  &\makecell{0.7912\\$\pm$0.0188}\\ \cline{2-10}
		                       &ARI &\makecell{0.7546\\$\pm$0.0487} &\makecell{0.6235\\$\pm$0.0010} &\makecell{0.6444 \\$\pm$0.0000}  &\makecell{0.4868\\$\pm$0.0829}
                                       &\makecell{\textbf{0.8715}\\$\pm$\textbf{0.0034}} &\makecell{0.6423\\$\pm$0.1030}  &\makecell{0.4743\\$\pm$0.0948}  &\makecell{\textbf{0.8181}\\$\pm$\textbf{0.0783}}\\ 
                                       \cline{2-10}
                                  &ACC &\makecell{0.8248\\$\pm$0.0310 } &\makecell{0.6832\\$\pm$0.0000} &\makecell{0.7228\\$\pm$0.0000}  &\makecell{0.6267\\$\pm$0.0536}
                                       &\makecell{\textbf{0.8703}\\$\pm$\textbf{0.0078}} &\makecell{0.7305\\$\pm$0.0831}  &\makecell{0.5958\\$\pm$0.0604}  &\makecell{\textbf{0.8507}\\$\pm$\textbf{0.0566}}\\ 
                                  \hline

            \multirow{3}{*}{Glass} 
            
                                  &NMI &\makecell{0.3140\\$\pm$0.0046} &\makecell{0.3073\\$\pm$0.0003} &\makecell{0.3120\\$\pm$0.0012}  &\makecell{0.0721\\$\pm$0.0069}
                                       &\makecell{0.3390\\$\pm$0.0549} &\makecell{\textbf{0.3424}\\$\pm$\textbf{0.0413}}  &\makecell{0.1654\\$\pm$0.0823}  &\makecell{\textbf{0.3764}\\$\pm$\textbf{0.0370}}\\ 
                                       \cline{2-10}
		                       &ARI &\makecell{0.1702\\$\pm$0.0026} &\makecell{0.1529\\$\pm$0.0005} &\makecell{0.1651\\$\pm$0.0009}  &\makecell{0.0070\\$\pm$0.0059}
                                       &\makecell{\textbf{0.1930}\\$\pm$\textbf{0.0435}} &\makecell{0.1648\\$\pm$0.0366}  &\makecell{0.0734\\$\pm$0.0659}  &\makecell{\textbf{0.1978}\\$\pm$\textbf{0.0193}}\\   
                                       \cline{2-10}
                                  &ACC &\makecell{0.4586\\$\pm$0.0076} &\makecell{0.4021\\$\pm$0.0011} &\makecell{0.4487\\$\pm$0.0007}  &\makecell{0.3368\\$\pm$0.0054}
                                       &\makecell{\textbf{0.4684}\\$\pm$\textbf{0.0476}} &\makecell{0.4202\\$\pm$0.0270}  &\makecell{0.4096\\$\pm$0.0584}  &\makecell{\textbf{0.4796}\\$\pm$\textbf{0.0086}}\\ 
                                  \hline

            \multirow{3}{*}{Ecoli} 
            
                                  &NMI &\makecell{\textbf{0.6379}\\$\pm$\textbf{0.0040}} &\makecell{0.5695\\$\pm$0.0115} &\makecell{0.6096\\$\pm$0.0076}  &\makecell{0.5012\\$\pm$0.0649}
                                       &\makecell{0.5985\\$\pm$0.0241} &\makecell{0.5541\\$\pm$0.0464}  &\makecell{0.4746\\$\pm$0.0424}  &\makecell{\textbf{0.6426}\\$\pm$\textbf{0.0024}}\\ 
                                       \cline{2-10}
		                       &ARI &\makecell{\textbf{0.5032}\\$\pm$\textbf{0.0070}} &\makecell{0.4142\\$\pm$0.0204} &\makecell{0.4815\\$\pm$0.0270}  &\makecell{0.3357\\$\pm$0.0752}
                                       &\makecell{0.4807\\$\pm$0.0898} &\makecell{0.3697\\$\pm$0.0862}  &\makecell{0.2832\\$\pm$0.0500}  &\makecell{\textbf{0.5157}\\$\pm$\textbf{0.0013}}\\   
                                       \cline{2-10}
                                  &ACC &\makecell{\textbf{0.6461}\\$\pm$\textbf{0.0103}} &\makecell{0.5769\\$\pm$0.0160} &\makecell{0.6238\\$\pm$0.0218}  &\makecell{0.5302\\$\pm$0.0435}
                                       &\makecell{0.6276\\$\pm$0.0711} &\makecell{0.5624\\$\pm$0.0638}  &\makecell{0.4598\\$\pm$0.0410}  &\makecell{\textbf{0.6482}\\$\pm$\textbf{0.0043}}\\ 
                                  \hline

            \multirow{3}{*}{Htru2} 
            
                                  &NMI &\makecell{0.4068\\$\pm$0.0000} &\makecell{\textbf{0.4100}\\$\pm$\textbf{0.0004}} &\makecell{0.4075\\$\pm$0.0002}  &\makecell{0.1289\\$\pm$0.0002}
                                       &\makecell{0.0204\\$\pm$0.0506} &\makecell{0.3671\\$\pm$0.0290}  &\makecell{0.0660\\$\pm$0.0268}  &\makecell{\textbf{0.5872}\\$\pm$\textbf{0.0003}}\\ 
                                       \cline{2-10}
		                       &ARI &\makecell{0.6071\\$\pm$0.0000} &\makecell{0.5796\\$\pm$0.0005} &\makecell{\textbf{0.6075}\\$\pm$\textbf{0.0002}}  &\makecell{0.0462\\$\pm$0.0002}
                                       &\makecell{0.0077\\$\pm$0.0649} &\makecell{0.5704\\$\pm$0.0379}  &\makecell{-0.0117\\$\pm$0.0074}  &\makecell{\textbf{0.7333}\\$\pm$\textbf{0.0002}}\\ 
                                       \cline{2-10}
                                  &ACC &\makecell{\textbf{0.9366}\\$\pm$\textbf{0.0000}} &\makecell{0.9252\\$\pm$0.0001} &\makecell{\textbf{0.9366}\\$\pm$\textbf{0.0000}}  &\makecell{0.6091\\$\pm$0.0002}
                                       &\makecell{0.8888\\$\pm$0.0637} &\makecell{0.9337\\$\pm$0.0126}  &\makecell{0.5170\\$\pm$0.0151}  &\makecell{\textbf{0.9661}\\$\pm$\textbf{0.0000}}\\ 
                                  \hline

            \multirow{3}{*}{Shill Bidding} 
                                  &NMI &\makecell{0.0021\\$\pm$0.0001} &\makecell{0.0023\\$\pm$0.0000} &\makecell{0.0042\\$\pm$0.0002}  &\makecell{0.0012\\$\pm$0.0000}
                                       &\makecell{0.0295\\$\pm$0.0359} &\makecell{\textbf{0.0351}\\$\pm$\textbf{0.0047}}  &\makecell{0.0105\\$\pm$0.0330}  &\makecell{\textbf{0.6216}\\$\pm$\textbf{0.0000}}\\ 
                                       \cline{2-10}
		                       &ARI &\makecell{-0.0006\\$\pm$0.0000} &\makecell{-0.0002\\$\pm$0.0000} &\makecell{-0.0002\\$\pm$0.0000}  &\makecell{-0.0002\\$\pm$0.0000}
                                       &\makecell{-0.0485\\$\pm$0.0482} &\makecell{-0.0854\\$\pm$0.0118}  &\makecell{\textbf{0.0079}\\$\pm$\textbf{0.0252}}  &\makecell{\textbf{0.7914}\\$\pm$\textbf{0.0000}}\\ 
                                       \cline{2-10}
                                  &ACC &\makecell{0.5020\\$\pm$0.0002} &\makecell{0.5059\\$\pm$0.0000} &\makecell{0.5090\\$\pm$0.0003}  &\makecell{0.5035\\$\pm$0.0002}
                                       &\makecell{0.7457\\$\pm$0.0961} &\makecell{\textbf{0.7869}\\$\pm$\textbf{0.0182}}  &\makecell{0.5937\\$\pm$0.0750}  &\makecell{\textbf{0.9674}\\$\pm$\textbf{0.0000}}\\ 
                                \hline

           \multirow{3}{*}{Anuran Calls}                                  
                                  &NMI &\makecell{\textbf{0.6775}\\$\pm$\textbf{0.0195}} &\makecell{0.5699\\$\pm$0.0002} &\makecell{0.6155\\$\pm$0.0022}  &\makecell{0.4270\\$\pm$0.0153}
                                       &\makecell{0.5575\\$\pm$0.0516} &\makecell{0.5178\\$\pm$0.0339}  &\makecell{0.5159\\$\pm$0.0247}  &\makecell{\textbf{0.6229}\\$\pm$\textbf{0.0134}}\\ 
                                       \cline{2-10}
		                       &ARI &\makecell{\textbf{0.5826}\\$\pm$\textbf{0.0174}} &\makecell{0.3414\\$\pm$0.0009} &\makecell{0.4017\\$\pm$0.0027}  &\makecell{0.2254\\$\pm$0.0694}
                                       &\makecell{0.4329\\$\pm$0.1363} &\makecell{0.2368\\$\pm$0.0694}  &\makecell{0.4268\\$\pm$0.0500}  &\makecell{\textbf{0.5145}\\$\pm$\textbf{0.0329}}\\ 
                                       \cline{2-10}
                                  &ACC &\makecell{\textbf{0.6441}\\$\pm$\textbf{0.0134}} &\makecell{0.4363\\$\pm$0.0009} &\makecell{0.5103\\$\pm$0.0073}  &\makecell{0.3705\\$\pm$0.0230}
                                       &\makecell{0.5682\\$\pm$0.0799} &\makecell{0.4101\\$\pm$0.0348}  &\makecell{0.5123\\$\pm$0.0271}  &\makecell{\textbf{0.5821}\\$\pm$\textbf{0.0420}}\\  
                                       \hline

            \multirow{3}{*}{Occupancy Detection}                                   
                                  &NMI &\makecell{\textbf{0.4890}\\$\pm$\textbf{0.0001}} &\makecell{0.4720\\$\pm$0.0003} &\makecell{0.4710\\$\pm$0.0002}  &\makecell{0.0539\\$\pm$0.0007}
                                       &\makecell{0.0934\\$\pm$0.1560} &\makecell{0.2982\\$\pm$0.1379}  &\makecell{0.0591\\$\pm$0.0464}  &\makecell{\textbf{0.5389}\\$\pm$\textbf{0.0002}}\\ 
                                       \cline{2-10}
		                       &ARI &\makecell{\textbf{0.5959}\\$\pm$\textbf{0.0001}} &\makecell{0.5699\\$\pm$0.0004} &\makecell{0.5670\\$\pm$0.0002}  &\makecell{0.0209\\$\pm$0.0002}
                                       &\makecell{0.0521\\$\pm$0.2248} &\makecell{0.2175\\$\pm$0.2226}  &\makecell{0.0363\\$\pm$0.0682}  &\makecell{\textbf{0.6750}\\$\pm$\textbf{0.0002}}\\ 
                                       \cline{2-10}
                                  &ACC &\makecell{\textbf{0.8912}\\$\pm$\textbf{0.0000}} &\makecell{0.8825\\$\pm$0.0001} &\makecell{0.8815\\$\pm$0.0001}  &\makecell{0.5772\\$\pm$0.0004}
                                       &\makecell{0.7597\\$\pm$0.0596} &\makecell{0.7026\\$\pm$0.1361}  &\makecell{0.6117\\$\pm$0.0595}  &\makecell{\textbf{0.9163}\\$\pm$\textbf{0.0001}}\\ 
                                       \hline

            \multirow{3}{*}{Machine Failure} 
            
                                  &NMI &\makecell{0.0248\\$\pm$0.0000} &\makecell{0.0281\\$\pm$0.0000} &\makecell{0.0310\\$\pm$0.0000}  &\makecell{0.0284\\$\pm$0.0003}
                                       &\makecell{\textbf{0.2285}\\$\pm$\textbf{0.0569}} &\makecell{0.0540\\$\pm$0.0292}  &\makecell{0.0137\\$\pm$0.0161}  &\makecell{\textbf{0.5825}\\$\pm$\textbf{0.1869}}\\ 
                                       \cline{2-10}
		                       &ARI &\makecell{-0.0116\\$\pm$0.0000} &\makecell{-0.0083\\$\pm$0.0000} &\makecell{-0.0052\\$\pm$0.0000}  &\makecell{0.0007\\$\pm$0.0001}
                                       &\makecell{\textbf{0.2217}\\$\pm$\textbf{0.0510}} &\makecell{-0.0037\\$\pm$0.0382}  &\makecell{0.0031\\$\pm$0.0083}  &\makecell{\textbf{0.6506}\\$\pm$\textbf{0.1987}}\\ 
                                       \cline{2-10}
                                  &ACC &\makecell{0.5673\\$\pm$0.0000} &\makecell{0.5418\\$\pm$0.0002} &\makecell{0.5202\\$\pm$0.0001}  &\makecell{0.5146\\$\pm$0.0007}
                                       &\makecell{\textbf{0.9653}\\$\pm$\textbf{0.0543}} &\makecell{0.6125\\$\pm$0.0412}  &\makecell{0.5481\\$\pm$0.0358}  &\makecell{\textbf{0.9863}\\$\pm$\textbf{0.0068}}\\ 
                                \hline

            \multirow{3}{*}{Pulsar Cleaned} 
                                  &NMI &\makecell{0.0311\\$\pm$0.0002} &\makecell{0.0167\\$\pm$0.0000} &\makecell{0.0201\\$\pm$0.0000}  &\makecell{0.0236 \\$\pm$0.0011}
                                       &\makecell{\textbf{0.0523}\\$\pm$\textbf{0.0813}} &\makecell{0.0044\\$\pm$0.0081}  &\makecell{0.0090\\$\pm$0.0067}  &\makecell{\textbf{0.0544}\\$\pm$\textbf{0.0001}}\\ 
                                       \cline{2-10}
		                       &ARI &\makecell{0.0370\\$\pm$0.0003} &\makecell{0.0069\\$\pm$0.0000} &\makecell{0.0131\\$\pm$0.0001}  &\makecell{-0.0011\\$\pm$0.0002}
                                       &\makecell{\textbf{0.0731}\\$\pm$\textbf{0.1193}} &\makecell{-0.0067\\$\pm$0.0160}  &\makecell{-0.0039\\$\pm$0.0080}  &\makecell{\textbf{0.1082}\\$\pm$\textbf{0.0002}}\\ 
                                       \cline{2-10}
                                  &ACC &\makecell{0.7329\\$\pm$0.0008} &\makecell{0.5714\\$\pm$0.0001} &\makecell{0.6171\\$\pm$0.0004}  &\makecell{0.5019\\$\pm$0.0012}
                                       &\makecell{\textbf{0.9488}\\$\pm$\textbf{0.0973}} &\makecell{\textbf{0.9043}\\$\pm$\textbf{0.1231}}  &\makecell{0.6036\\$\pm$0.0659}  &\makecell{0.8834\\$\pm$0.0002}\\ 
                                       \hline

            \multirow{3}{*}{Bert-Embedded Spam}
                                  &NMI &\makecell{\textbf{0.1517}\\$\pm$\textbf{0.0001}} &\makecell{0.1249\\$\pm$0.0000} &\makecell{0.0472\\$\pm$0.0374}  &\makecell{0.1122\\$\pm$0.0023}
                                       &\makecell{0.0330\\$\pm$0.0791} &\makecell{0.0431\\$\pm$0.0946}  &\makecell{0.0238\\$\pm$0.0004}  &\makecell{\textbf{0.6704}\\$\pm$\textbf{0.0000}}\\ 
                                       \cline{2-10}
		                       &ARI &\makecell{\textbf{0.0830}\\$\pm$\textbf{0.0003}} &\makecell{0.0550\\$\pm$0.0000} &\makecell{0.0289\\$\pm$0.0265}  &\makecell{-0.0026\\$\pm$0.0007}
                                       &\makecell{0.0080\\$\pm$0.1084} &\makecell{0.0296\\$\pm$0.1287}  &\makecell{-0.0669\\$\pm$0.0004}  &\makecell{\textbf{0.8216}\\$\pm$\textbf{0.0000}}\\ 
                                       \cline{2-10}
                                  &ACC &\makecell{0.6453\\$\pm$0.0002} &\makecell{0.6174\\$\pm$0.0000} &\makecell{0.5837\\$\pm$0.0390}  &\makecell{0.5482\\$\pm$0.0012}
                                       &\makecell{0.8210\\$\pm$0.0816} &\makecell{\textbf{0.8281}\\$\pm$\textbf{0.0832}}  &\makecell{0.8038\\$\pm$0.0002}  &\makecell{\textbf{0.9664}\\$\pm$\textbf{0.0000}}\\ 
                                       \hline
            
            \hline
            
	\end{tabular}
    }
	}
 }
\end{table*}

\begin{figure*}[!t]
\centering
    \subfloat[]{\includegraphics[width=0.23\textwidth]{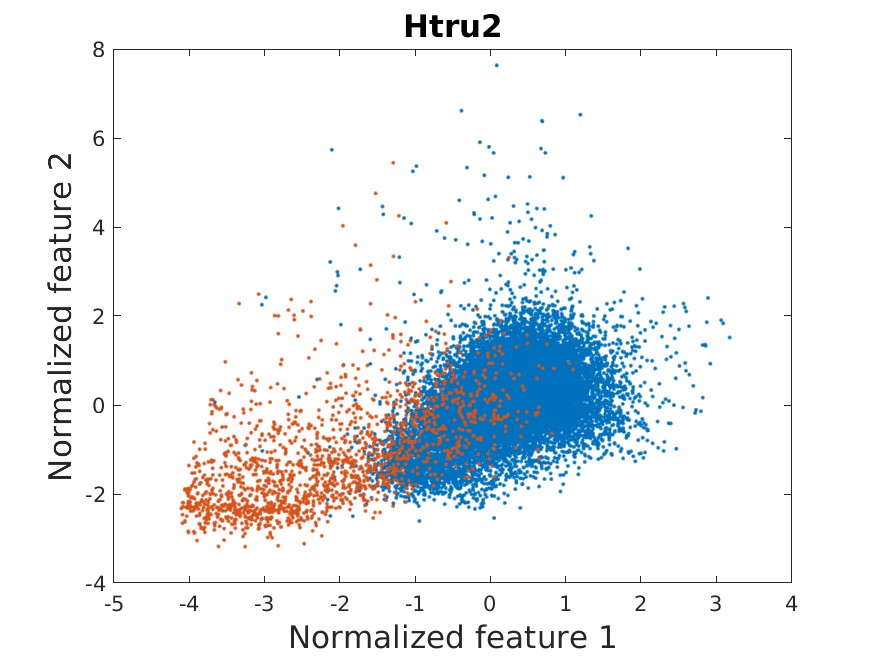}}
    \subfloat[]{\includegraphics[width=0.23\textwidth]{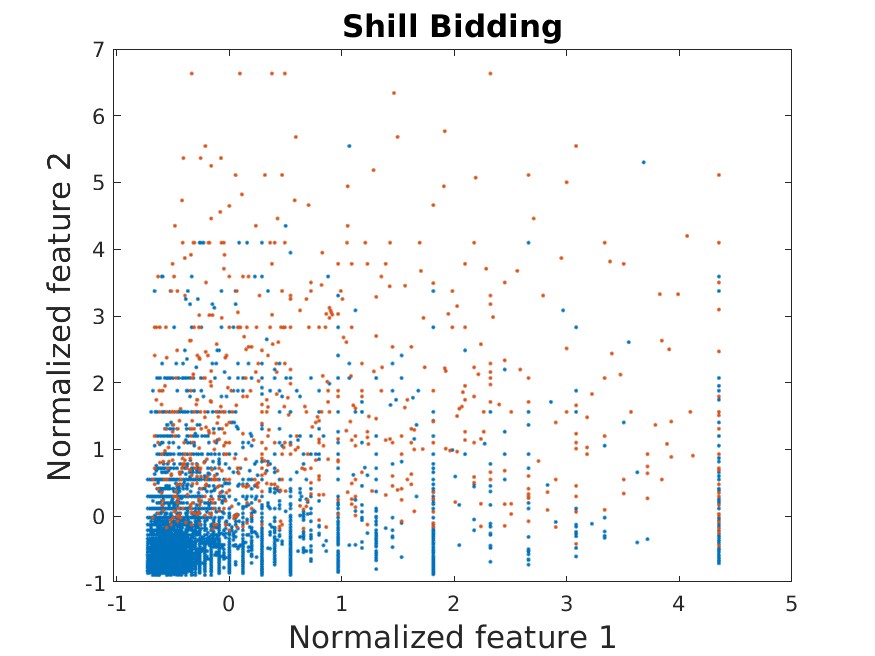}}
    \subfloat[]{\includegraphics[width=0.23\textwidth]{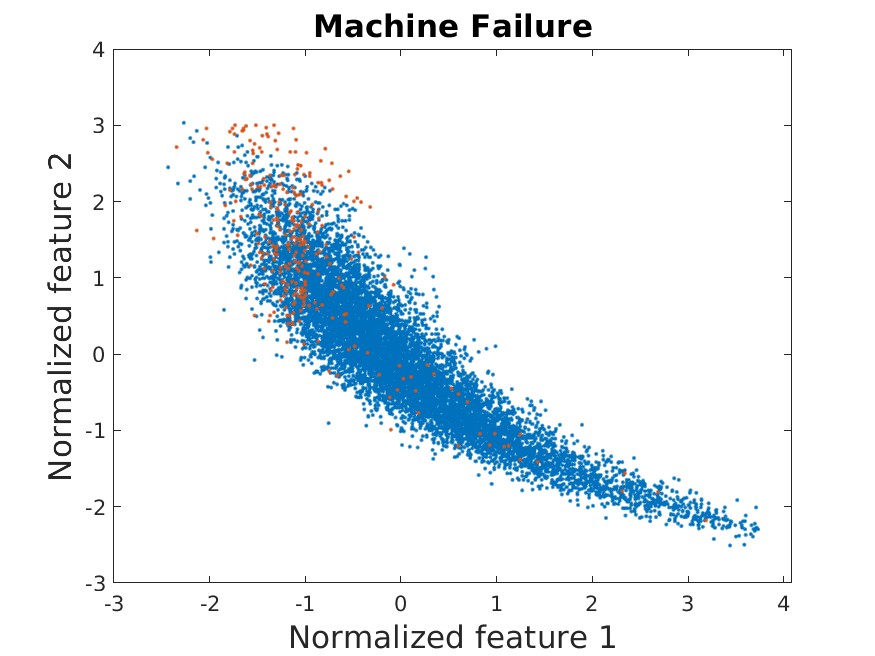}}
    \subfloat[]{\includegraphics[width=0.23\textwidth]{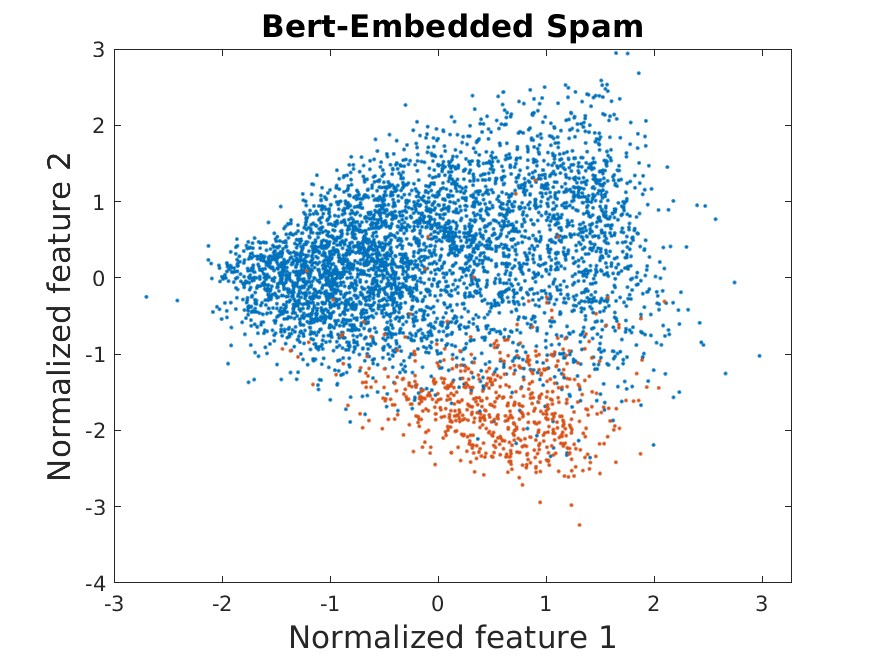}}\\
    \subfloat[]{\includegraphics[width=0.23\textwidth]{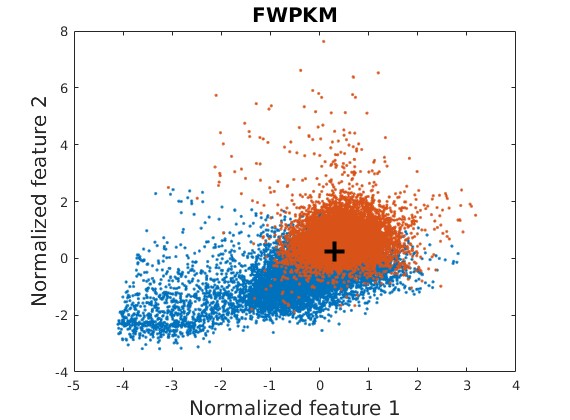}\label{fig:htru2_fwpkm}}
    \subfloat[]{\includegraphics[width=0.23\textwidth]{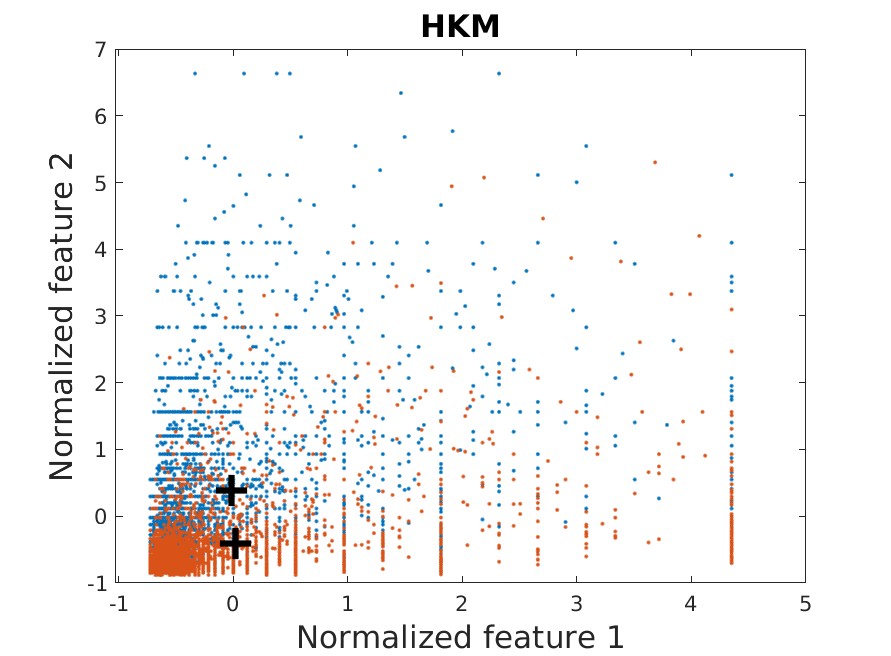}\label{fig:shill_hkm}}
    \subfloat[]{\includegraphics[width=0.23\textwidth]{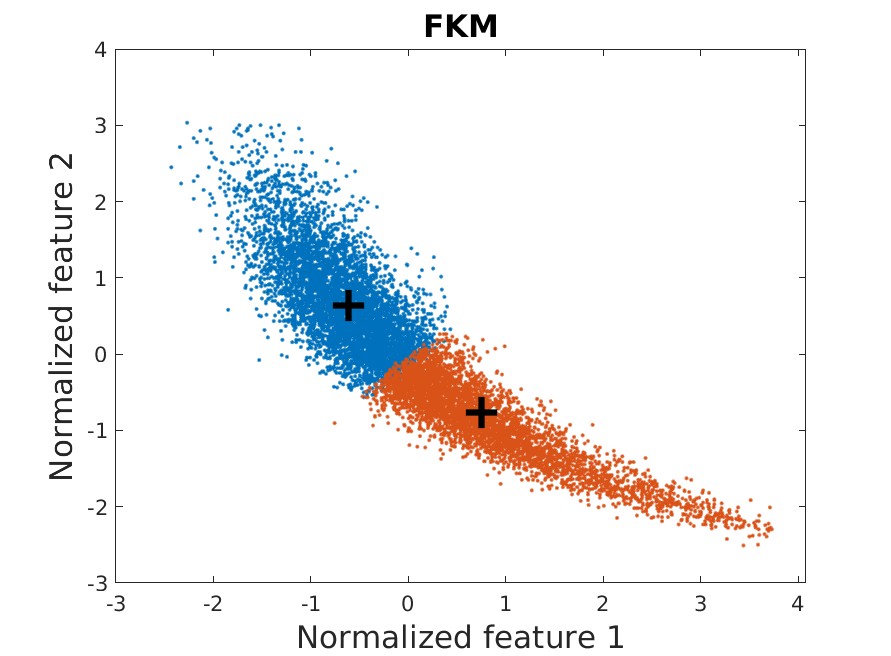}\label{fig:machine_fkm}}
    \subfloat[]{\includegraphics[width=0.23\textwidth]{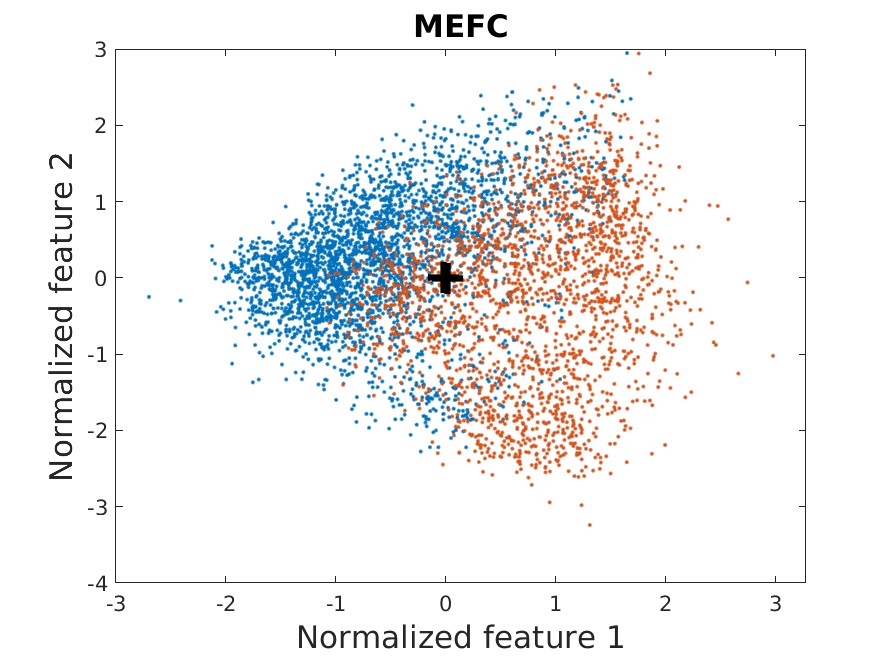}\label{fig:spam_mefc}}\\
    \subfloat[]{\includegraphics[width=0.23\textwidth]{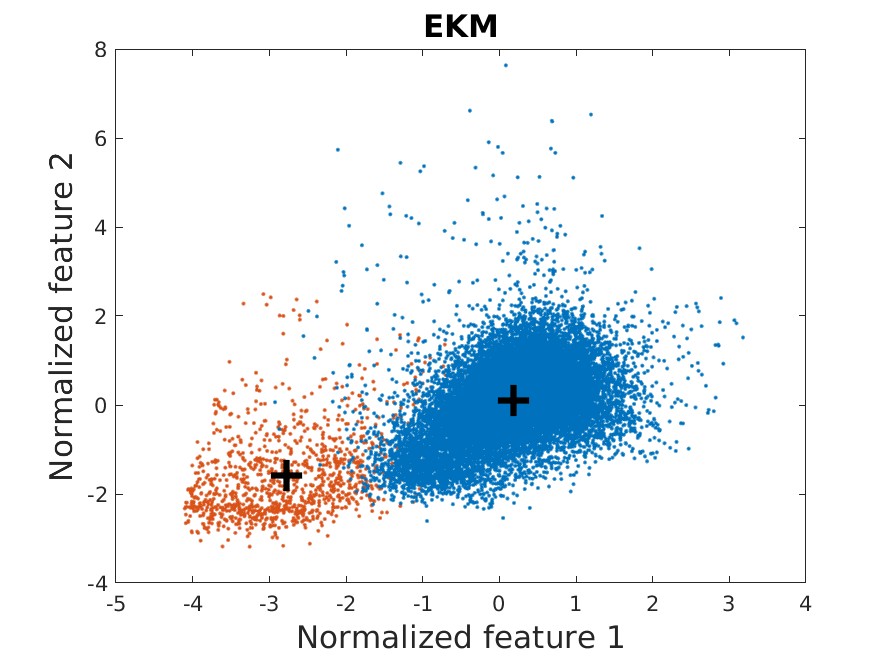}\label{fig:htru2_ekm}}
    \subfloat[]{\includegraphics[width=0.23\textwidth]{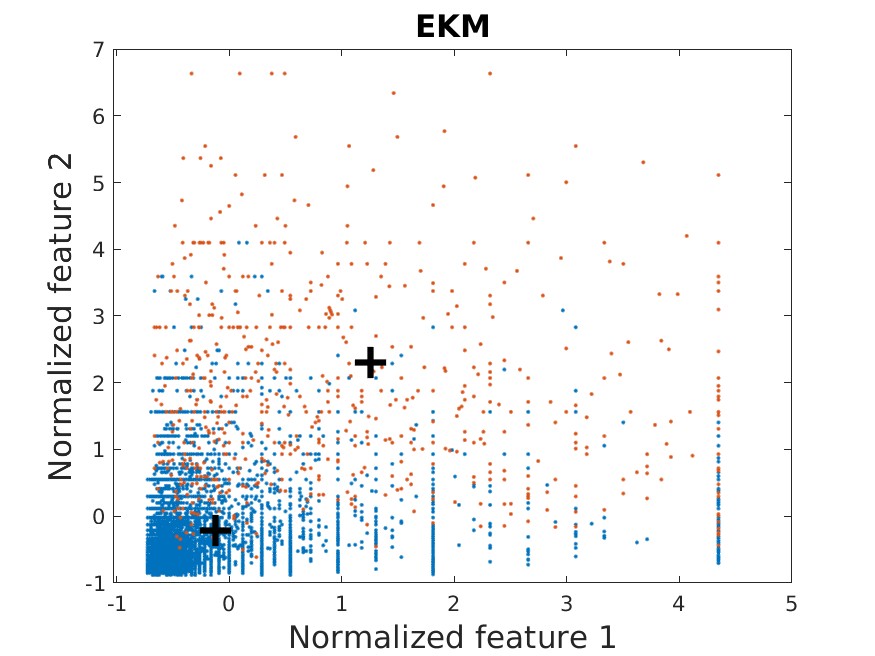}\label{fig:shill_ekm}}
    \subfloat[]{\includegraphics[width=0.23\textwidth]{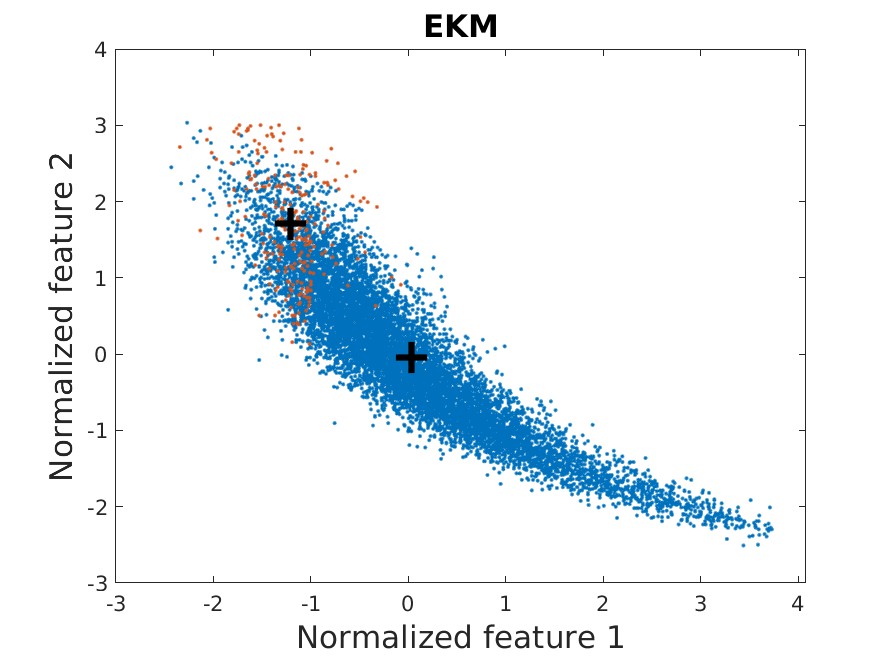}\label{fig:machine_ekm}}
    \subfloat[]{\includegraphics[width=0.23\textwidth]{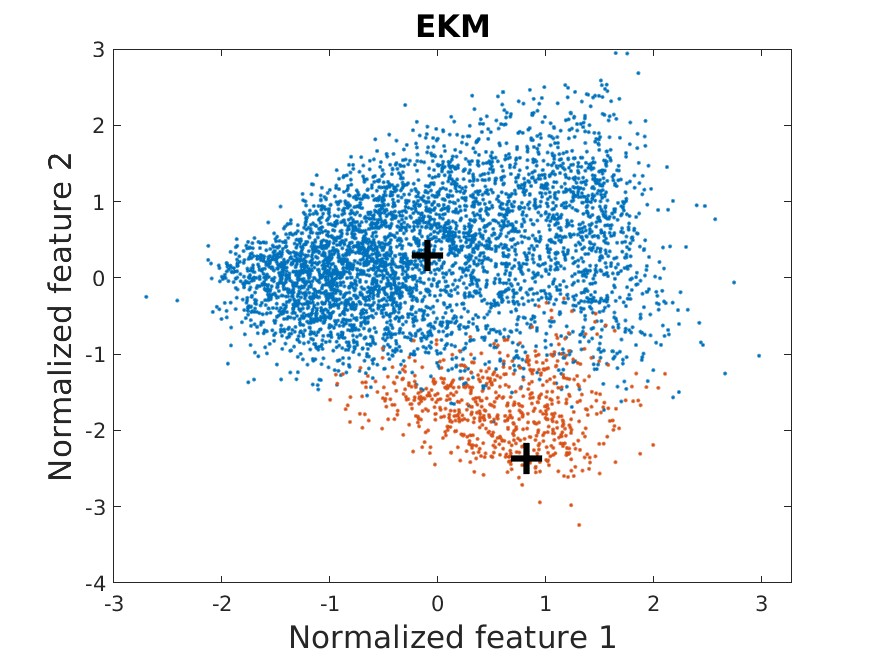}\label{fig:spam_ekm}}

 \caption{Scatter diagrams of some selected real imbalanced datasets (from left to right: Htru2, Shill Bidding, Machine Failure, and Bert-Embedded Spam). The first two features are used as coordinates for visualization. (a)-(d) Primitive scatter diagrams with reference class labels. (e)-(f) Clustering results of some benchmark algorithms. (i)-(l) Clustering results of the proposed EKM. The black crosses represent centroids obtained by each algorithm.}
\label{fig:imbalanced_datasets}
\end{figure*}
\subsection{Empirical Convergence and Algorithm Efficiency}
\label{sect5:convergence}
Table~\ref{table:iter_time} shows the average number of iterations and the average time per run (each algorithm is run 5000 times). The numbers in brackets in the table are the ranking numbers of the corresponding algorithms. AVK and MDK give the average and median number of iterations and computational time of each algorithm on all test datasets, respectively.

As evident from Table~\ref{table:iter_time}, AVK and MDK are consistent except that EKM is faster than FKM in AVK and slower in MDK. This is because FKM is slow on the Bert-Embedded Spam dataset. HKM is the fastest algorithm as it does not require membership calculation. MEFC comes second, followed by FKM and EKM. EKM is 30\% faster than FKM in AVK and slower than FKM by 9\% in MDK. In terms of absolute speed, MEFC, FKM, and EKM operate within the same order of magnitude. This aligns with their identical time complexity of $O(NK^2P)$ per iteration, where $N$ is the instance number, $K$ is the cluster number, and $P$ is the data dimension or feature number. The table also demonstrates the empirical convergence of EKM, averaging 44 iterations or a median of 20 iterations to reach the specified convergence~\eqref{convergence_condition}. In summary, EKM has a low computational cost.

\begin{table*}[!t]
\caption{
Average number of iterations and average calculation time of each algorithm
\label{table:iter_time}}
\centering
{%
    \resizebox{0.95\textwidth}{!}
    {
        {
        \begin{tabular}{cccccccccc} \hline
		Dataset&Measurement&HKM &FKM &MEFC  &PFKM&csiFKM &siibFKM  &FWPKM&EKM   \\ \hline
            \multirow{2}{*}{Data-A} 
                                &Iter &17.3     &13.7     &29.0     &33.5     &30.7     &28.0     &44.6     &12.1\\  
                                &Time &0.005(1) &0.008(3) &0.013(4) &0.033(7) &0.021(6) &0.277(8) &0.016(5) &0.007(2)\\ \hline
            \multirow{2}{*}{Data-B}
                                &Iter &9.2      &13.2     &30.1     &28.8     &15.1     &17.4     &24.1     &17.5\\  
                                &Time &0.003(1) &0.008(2) &0.013(6) &0.031(7) &0.011(5) &0.168(8) &0.010(4) &0.009(3)\\ \hline
            \multirow{2}{*}{Data-C}
                                &Iter &16.3     &17.6     &36.0     &23.9     &35.0     &42.1     &73.7     &23.5\\  
                                &Time &0.008(1) &0.014(2) &0.023(4) &0.048(7) &0.040(6) &0.624(8) &0.035(5) &0.017(3)\\ \hline
            \multirow{2}{*}{Data-D}
                                &Iter &14.8     &23.6     &25.4    &39.8      &24.5     &25.2     &7.3      &15.6\\  
                                &Time &0.033(2) &0.091(5) &0.082(4) &0.364(7) &0.142(6) &1.867(8) &0.030(1) &0.058(3)\\ \hline
            \multirow{2}{*}{IS}
                                &Iter &15.9     &40.1     &20.9     &53.0     &49.8     &32.1     &29.6    &24.1\\  
                                &Time &0.021(1) &0.110(5) &0.058(2) &0.218(7) &0.104(4) &0.824(8) &0.152(6) &0.074(3)\\ \hline
            \multirow{2}{*}{Seeds}
                                &Iter &8.5      &10.4     &11.1     &13.3     &12.6     &26.0     &32.5     &12.0\\  
                                &Time &0.001(1) &0.001(3) &0.001(2) &0.003(6) &0.002(4) &0.028(8) &0.005(7) &0.002(5)\\ \hline
            \multirow{2}{*}{Heart Disease}
                                &Iter &9.9      &269.3    &12.5     &22.7     &500.0    &498.8    &3.3      &18.7\\  
                                &Time &0.002(1) &0.074(5) &0.003(3) &0.081(6) &2.095(8) &1.646(7) &0.002(2) &0.005(4)\\ \hline
            \multirow{2}{*}{Wine}
                                &Iter &7.0      &13.5     &10.6     &25.3     &15.7     &40.4     &19.4     &11.7\\  
                                &Time &0.000(1) &0.002(4) &0.001(2) &0.004(6) &0.002(5) &0.037(8) &0.005(7) &0.001(3)\\ \hline
            \multirow{2}{*}{Rice}
                                &Iter &8.4      &8.7      &8.9      &12.7     &14.0     &17.8     &72.8     &9.1\\  
                                &Time &0.004(1) &0.009(4) &0.006(2) &0.021(6) &0.011(5) &0.214(8) &0.072(7) &0.009(3)\\ \hline
            \multirow{2}{*}{WDBC}
                                &Iter &7.4      &10.8     &10.6     &32.4     &500.0    &500.0    &16.2     &9.9\\  
                                &Time &0.001(1) &0.003(4) &0.002(2) &0.009(5) &2.307(8) &0.933(7) &0.013(6) &0.002(3)\\ \hline
            \multirow{2}{*}{Zoo}
                                &Iter &4.4      &31.0     &23.8     &35.0     &34.4     &37.8    &2.3      &498.9\\  
                                &Time &0.001(1) &0.005(4) &0.003(3) &0.009(6) &0.006(5) &0.048(7) &0.001(2) &0.069(8)\\ \hline
            \multirow{2}{*}{Glass}
                                &Iter &8.5      &27.2     &29.4     &18.8     &31.4     &45.0    &25.7     &25.3\\  
                                &Time &0.001(1) &0.006(5) &0.006(2) &0.009(7) &0.006(4) &0.096(8) &0.007(6) &0.006(3)\\ \hline
            \multirow{2}{*}{Ecoli}
                                &Iter &14.1     &42.9     &23.6     &32.2     &40.5     &45.2     &3.0      &22.9\\  
                                &Time &0.002(2) &0.019(6) &0.010(3) &0.029(7) &0.015(5) &0.197(8) &0.002(1) &0.011(4)\\ \hline
            \multirow{2}{*}{Htru2}
                                &Iter &12.7     &33.8    &15.2     &21.5      &500.0     &500.0     &52.5     &10.8\\  
                                &Time &0.046(1) &0.175(4) &0.067(3) &0.357(5) &17.884(7) &27.836(8) &0.450(6) &0.059(2)\\ \hline
            \multirow{2}{*}{Shill Bidding}
                                &Iter &12.0     &15.0     &24.3     &22.7     &36.8    &40.4      &25.6     &22.3\\  
                                &Time &0.011(1) &0.026(2) &0.029(3) &0.064(7) &0.046(5) &0.794(8) &0.056(6) &0.033(4)\\ \hline
            \multirow{2}{*}{Anuran Calls}
                                &Iter &22.5     &35.7     &24.1     &39.3     &42.6     &45.4     &24.0     &44.3\\  
                                &Time &0.339(1) &0.527(3) &0.377(2) &1.502(6) &0.614(4) &5.119(8) &2.089(7) &0.719(5)\\ \hline
            \multirow{2}{*}{Occupancy Detection}
                                &Iter &11.5     &25.6     &28.6     &24.6     &28.0     &194.9     &19.8     &14.8\\  
                                &Time &0.032(1) &0.132(5) &0.097(3) &0.291(7) &0.109(4) &12.090(8) &0.163(6) &0.067(2)\\ \hline
            \multirow{2}{*}{Machine Failure}
                                &Iter &13.4     &22.0     &11.9     &127.4    &422.1    &34.5     &122.6     &22.5\\  
                                &Time &0.015(1) &0.055(4) &0.022(2) &0.442(6) &0.724(7) &1.058(8) &0.331(5)  &0.049(3)\\ \hline
            \multirow{2}{*}{Pulsar Cleaned}
                                &Iter &25.3     &17.1     &18.4     &60.7     &500.0    &500.0     &50.2     &21.4\\  
                                &Time &0.045(1) &0.068(3) &0.052(2) &0.435(6) &12.560(7) &22.953(8) &0.268(5) &0.075(4)\\ \hline
            \multirow{2}{*}{Bert-Embedded Spam}
                                &Iter &31.6     &500.0    &498.7    &21.2     &500.0    &500.0    &28.5     &41.6\\  
                                &Time &0.015(1) &0.618(5) &0.388(4) &0.642(6) &4.866 (7) &8.520(8) &0.030(2) &0.045(3)\\ \hline \hline
            \multirow{2}{*}{AVK}
                                &Iter &13.5     &58.5     &44.7     &34.4     &166.4    &158.3    &33.9     &44.0\\  
                                &Time &0.029(1) &0.098(4) &0.063(2) &0.230(6) &2.078(7) &4.266(8) &0.187(5) &0.066(3)\\ \hline
            \multirow{2}{*}{MDK}
                                &Iter &12.4     &22.8     &23.7     &27.1      &35.9     &41.3     &25.6    &20.0\\  
                                &Time &0.007(1) &0.023(3) &0.018(2) &0.056(6) &0.075(7) &0.809(8) &0.030(5) &0.025(4)\\ \hline
	\end{tabular}
 
        }
   }
 }
\end{table*}

\section{Deep Clustering}
\label{sect6}
Clustering algorithms using Euclidean distance for defining similarity face a challenge when clustering high-dimensional data, such as images, due to the distance differences between point pairs vanishes in high-dimensional space~\cite{beyer1999nearest}. Deep clustering is a technique that addresses this issue by employing DNNs to map high-dimensional data to low-dimensional representation. Many deep clustering methods combine DNNs with HKM, e.g.,~\cite{yang2017towards,caron2018deep,fard2020deep}. However, this design is less favorable to imbalanced data~\cite{caron2018deep}. In this section, we show that EKM is a better alternative for deep clustering of imbalanced data.

\subsection{Optimization Procedure}
\begin{figure}[t!]
    \centering
    \includegraphics[width=0.45\textwidth]{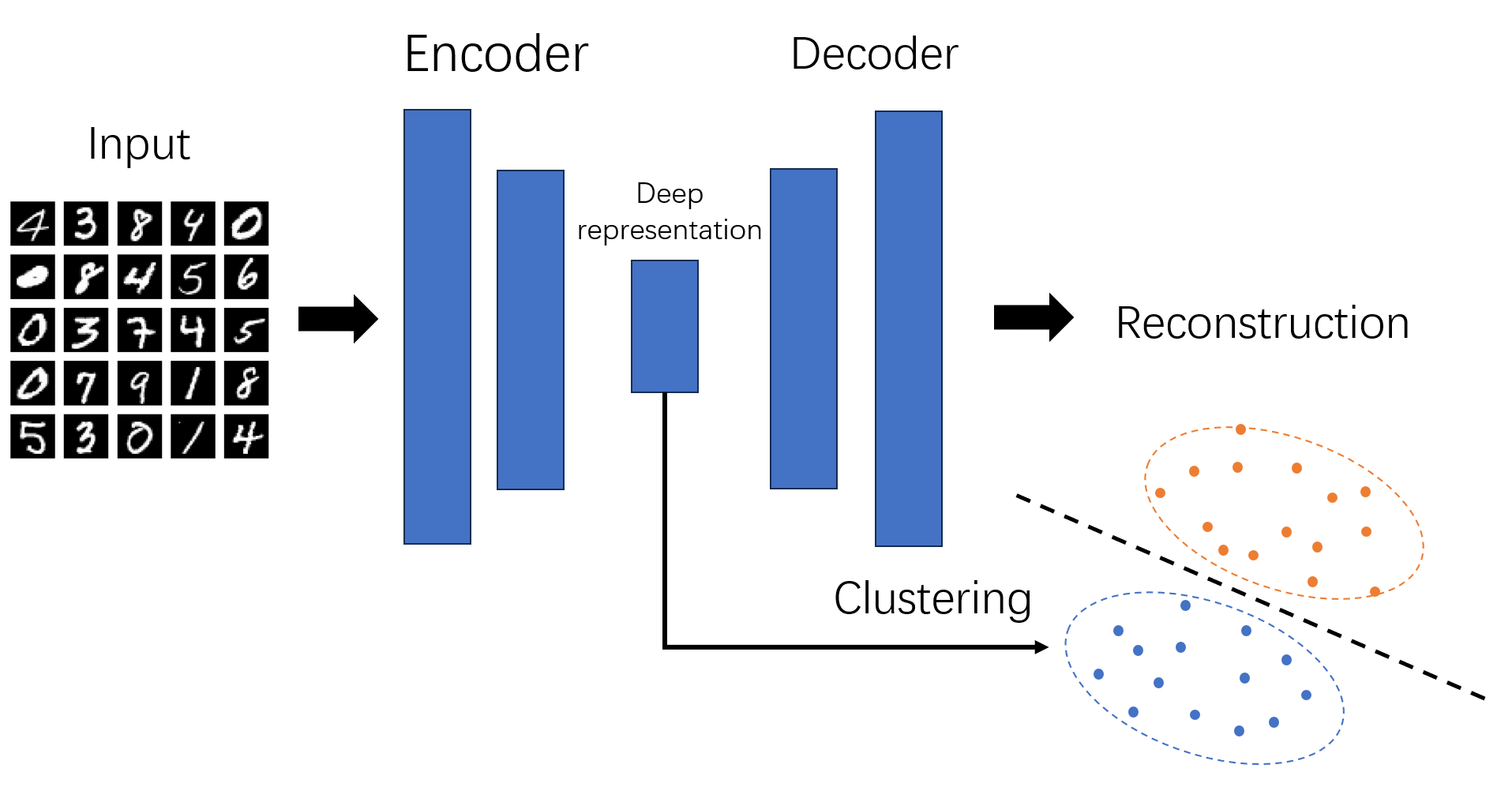}
    \caption{Illustration of the DCN framework~\cite{yang2017towards}. The parameters of the encoder, decoder, and clustering model are optimized jointly to minimize the reconstruction error and clustering error.}
    \label{fig:dcn_illustration}
\end{figure}

Deep clustering network (DCN)~\cite{yang2017towards} is a popular deep clustering framework. As illustrated in Fig.~\ref{fig:dcn_illustration}, DCN maps high-dimensional data to low-dimensional representation through an autoencoder network. An autodecoder follows the autoencoder, mapping the representation back to the original high-dimensional space (i.e., reconstruction). To ensure that the low-dimensional representation maintains the primary information of the original data, the autoencoder and the autodecoder are jointly trained to minimize the reconstruction error. Additionally, to make the low-dimensional representation have a clustering-friendly structure, a clustering error is minimized along with the reconstruction error. DCN uses an alternating optimization algorithm to minimize the total error, and the optimization process is described below.

First, the autoencoder and the autodecoder are jointly trained to reduce the following loss for the incoming data $\mathbf{x}_n$:
\begin{equation}
\label{DCNloss_encoder_decoder}
    \min_{\bm{\theta}_e,\bm{\theta}_d} L^n=l(\bm{g}(\bm{f}(\mathbf{x}_n)),\mathbf{x}_n)+\beta \|\bm{f}(\mathbf{x}_n)-\mathbf{C}\mathbf{s}_n\|_2^2,
\end{equation}
where $\bm{f}(\cdot)$ and $\bm{g}(\cdot)$ are simplified symbols for autoencoder $\bm{f}(\cdot;\,\bm{\theta}_e)$ and autodecoder $\bm{g}(\cdot;\,\bm{\theta}_d)$, respectively. The function $l(\cdot)$ is the least-squares loss $l(\hat{\mathbf{x}},\mathbf{x})=\|\hat{\mathbf{x}}-\mathbf{x}\|_2^2$ to measure the reconstruction error. The assignment vector $\mathbf{s}_n\in\mathbb{R}^{K\times 1}$ has only one non-zero element and $\mathbf{1}\trans \mathbf{s}_n=1$, indicating which cluster the $n$-th data belongs to, and the $k$-th column of $\mathbf{C}=[\mathbf{c}_1,\cdots,\mathbf{c}_K]$ is the centroid of the $k$-th cluster. The parameter $\beta$ balances the reconstruction error versus the clustering error. Then, the network parameters $\{\bm{\theta}_e,\bm{\theta}_d\}$ are fixed, and the parameters $\{\mathbf{s}_n\}$ are updated as follows:
\begin{equation}
\label{DCN_update_assignment}
    s_{j,n}\leftarrow
    \begin{cases}
        1, & \text{if}\, j=\argmin_{k=\{1,\cdots,K\}} \|\bm{f}(\mathbf{x}_n)-\mathbf{c}_k\|_2,\\
        0, & \text{otherwise,}
    \end{cases}
\end{equation}
where $s_{j,n}$ is the $j$-th element of $\mathbf{s}_n$. Finally, $\mathbf{C}$ is updated by the batch-learning version of the HKM algorithm:
\begin{equation}
\label{DCN_KM_update_centroids}
    \mathbf{c}_k \leftarrow \mathbf{c}_k + (1/m_k^n)(\bm{f}(\mathbf{x}_n)-\mathbf{c}_k)s_{k,n},
\end{equation}
where $m_k^n$ is the number of samples assigned to the $k$-th cluster before the incoming data $\mathbf{x}_n$, controlling the learning rate of the $k$-th centroid. Overall, the optimization procedure of DCN alternates between updating network parameters $\{\bm{\theta}_e,\bm{\theta}_d\}$ by solving~\eqref{DCNloss_encoder_decoder} and updating HKM parameters $\{\mathbf{C},\{\mathbf{s}_n\}\}$ by~\eqref{DCN_update_assignment} and~\eqref{DCN_KM_update_centroids}.

However, the centroid updating rule~\eqref{DCN_KM_update_centroids} is problematic for imbalanced data due to the uniform effect. To address this issue, we propose to replace~\eqref{DCN_KM_update_centroids} with the batch-learning version of EKM:
\begin{equation}
\label{DCN_EKM_update_centroids}
\begin{aligned}
    \mathbf{c}_k &\leftarrow \mathbf{c}_k - (1/m_k^n)\partial_{\mathbf{c}_k}J_B(\mathbf{c}_1,\cdots,\mathbf{c}_K)\\
    & = \mathbf{c}_k+(1/m_k^n)\frac{e^{-\alpha d_{kn}}}{\sum_{i=1}^K e^{-\alpha d_{in}}} \big[1-\alpha(d_{kn}\\
    &-\frac{\sum_{i=1}^K d_{in}e^{-\alpha d_{in}}}{\sum_{i=1}^K e^{-\alpha d_{in}}})\big](\bm{f}(\mathbf{x}_n)-\mathbf{c}_k),
\end{aligned}
\end{equation}
where $d_{kn}=\frac{1}{2}\|\bm{f}(\mathbf{x}_n)-\mathbf{c}_k\|_2^2$. There are other details and tricks to implement DCN, such as the initialization of the networks. We only introduce the part related to our contribution, and kindly refer to~\cite{yang2017towards} for more implementation details.

\subsection{Clustering Performance on MNIST}
To implement DCN, we refer to the code one of its authors provided, available at \url{https://github.com/boyangumn/DCN-New}. We use the default neural network structure and hyperparameters. In particular, the dimension of the low-dimensional representation is set to ten, and the parameter $\beta$ is set to one. The smoothing parameter $\alpha$ is tuned on the representation obtained by the initialized DCN network according to the strategy introduced in Section~\ref{sect3_alpha}.

We first evaluate the algorithm's performance on a balanced dataset. The considered dataset is the full MNIST~\cite{lecun1998mnist}, which contains 70,000 gray images of handwritten digits from 0 to 9. Each digit has approximately 7,000 images and each image has $28 \times 28=784$ pixels. We set the smoothing parameter of EKM to $\alpha=5\mathrm{e}-3$. The clustering results are presented in TABLE~\ref{table: eval_fullmnist}. We compare the proposed DCN+EKM with DCN+HKM and stacked autoencoder (SAE). SAE is a specific version of DCN that only minimizes the reconstruction error. Thus, the learned representation by SAE does not have a clustering-friendly structure. The results show that DCN outperforms SAE, highlighting the importance of a clustering-friendly structure. We can also see that DCN+EKM performs similarly to DCN+HKM. We omit the results of DCN+FKM and DCN+MEFC due to save space. Their performance is not better than DCN+HKM.

Then we evaluate the algorithm performance on an imbalanced dataset. The considered dataset is derived from the full MNIST by removing the training images of digits 1 to 9. This imbalanced dataset (we call imbalanced MNIST) contains 15,923 images, of which approximately 7,000 are digit 0, while the remaining digits (1 to 9) each have about 1,000 images. We set the smoothing parameter of EKM to $\alpha=3.8\mathrm{e}-3$. The results are summarized in TABLE~\ref{table: eval_reducedmnist}. We can see that DCN+EKM greatly outperforms other algorithms. This finding implies that the EKM-friendly structure is crucial for deep clustering of imbalanced data. 

We map the ten-dimensional representation obtained by DCN to a two-dimensional space by t-SNE~\cite{van2008visualizing} for visualization. The representation learned on the full MNIST dataset is displayed in Fig.~\ref{fig:tsne2_DCN_fullmnist} and the representation learned on the imbalanced MNIST dataset is displayed in Fig.~\ref{fig:tsne2_DCN_imbalancedMNIST}. The visualization results suggest that the deep representation obtained by DCN+EKM is more discriminative than that obtained by DCN+HKM as the former has clearer inter-cluster margins, especially for those obtained from the imbalanced MNIST dataset. This should be attributed to the centroid repulsion mechanism of EKM, which makes data farther away in representation space. We also observe from Fig.~\ref{fig:tsne2_DCN_imbalancedMNIST} that EKM successfully identifies the large class (digit 0) from other small classes without referring to the true labels, while HKM incorrectly divides the large class into four clusters to balance the cluster sizes.

\begin{table}[!t]
\caption{
Evaluation on full MNIST
\label{table: eval_fullmnist}}
\centering
\adjustbox{}{%
	\begin{tabular}{c|cccc}\hline
		Methods& 
		SAE+HKM &DCN+HKM &SAE+EKM &DCN+EKM\\ \hline
		NMI &0.725& 0.798& 0.711&  \textbf{0.813}\\   
		ARI &0.667&\textbf{0.744}& 0.642&  0.731\\   
		ACC &0.795& \textbf{0.837}& 0.782&  0.808\\  \hline
	\end{tabular}
	}
\end{table}

\begin{table}[!t]
\caption{
Evaluation on imbalanced MNIST
\label{table: eval_reducedmnist}}
\centering
\adjustbox{}{%
	\begin{tabular}{c|cccc}\hline
		Methods& 
		SAE+HKM &DCN+HKM &SAE+EKM &DCN+EKM\\ \hline
		NMI &0.551 &0.584 &0.583 &\textbf{0.701}\\   
		ARI &0.317 &0.325 &0.396 &\textbf{0.826}\\   
		ACC &0.413 &0.434 &0.497 &\textbf{0.784}\\  \hline
	\end{tabular}
	}
\end{table}

\begin{figure}[!t]
\centering
 \subfloat[]{\includegraphics[width=0.22\textwidth]{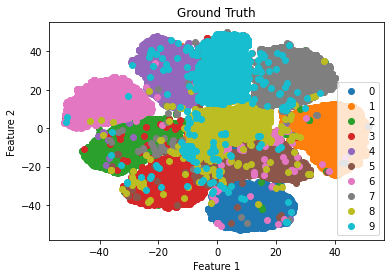}}
 \subfloat[]{\includegraphics[width=0.22\textwidth]{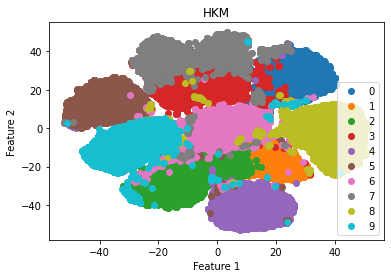}}
  \hfill
\subfloat[]{\includegraphics[width=0.22\textwidth]{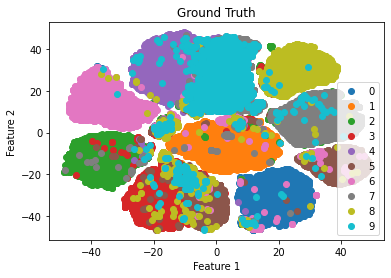}}
\subfloat[]{\includegraphics[width=0.22\textwidth]{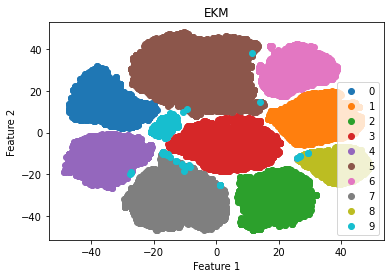}}
 \caption{The t-SNE visualization of DCN representations learned on the full MNIST dataset. (a) Representations obtained by DCN+HKM, colored by reference labels. (b)  Representations obtained by DCN+HKM, colored by labels generated by HKM. (c) Representations obtained by DCN+EKM, colored by reference labels. (d) Representations obtained by DCN+EKM, colored by labels generated by EKM.}
\label{fig:tsne2_DCN_fullmnist}
\end{figure}

\begin{figure}[!t]
\centering
 \subfloat[]{\includegraphics[width=0.22\textwidth]{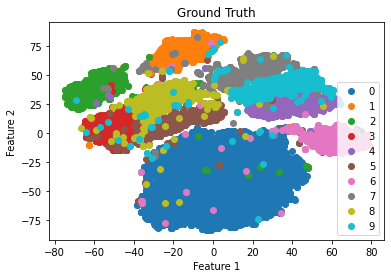}
 \label{fig:tsne2_true_label_km}}
 \subfloat[]{\includegraphics[width=0.22\textwidth]{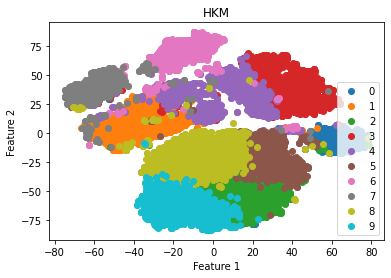}
 \label{fig:tsne2_est_label_km}}
 \hfill
\subfloat[]{\includegraphics[width=0.22\textwidth]{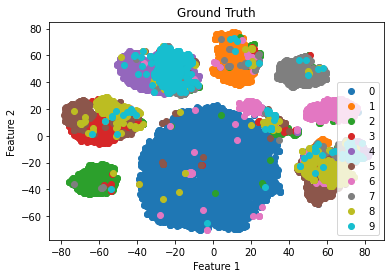}
\label{fig:tsne2_true_label_ekm}}
\subfloat[]{\includegraphics[width=0.22\textwidth]{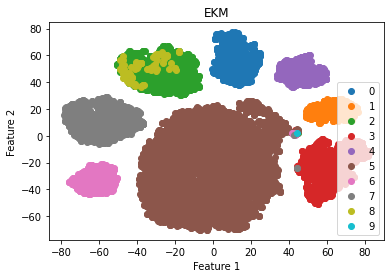}
\label{fig:tsne2_est_label_ekm}}
 \caption{The t-SNE visualization of DCN representations learned on the imbalanced MNIST dataset. (a) Representations obtained by DCN+HKM, colored by reference labels. (b)  Representations obtained by DCN+HKM, colored by labels generated by HKM. (c) Representations obtained by DCN+EKM, colored by reference labels. (d) Representations obtained by DCN+EKM, colored by labels generated by EKM.}
\label{fig:tsne2_DCN_imbalancedMNIST}
\end{figure}

\section{Conclusion}
\label{sect7}
This paper presents equilibrium K-means (EKM), a novel clustering algorithm with a new centroid repulsion mechanism effective for imbalanced data. EKM is simple, interpretable, and scalable to large datasets. Experimental results on datasets from various domains show that EKM outperforms HKM, FKM, and other state-of-the-art centroid-based algorithms on imbalanced datasets and performs comparably on balanced datasets. We also demonstrate that EKM is a better alternative to HKM and FKM in deep clustering when dealing with imbalanced data. Furthermore, we reformulate HKM, FKM, and EKM in a general form of gradient descent. We encourage readers to study more properties of this general form. Combining EKM with kernels, multi-prototype mechanisms, and other techniques to deal with more complex data structures is also an important research direction. We believe that these explorations will promote the development of EKM and further solve the data imbalance issue.
\begin{appendices}
\section{Proof of Theorem~\ref{theorem_convergence}}
For concise and tidy, we denote the partial derivative $\frac{\partial h}{\partial d_k}(d_{1n}^{(\tau)},\cdots,d_{Kn}^{(\tau)})$ as $\frac{\partial h}{\partial d_k}\vert_{n,{\tau}}$, $h(d_{1n},\cdots,d_{Kn})$ as $h$, and $h(d_{1n}^{(\tau)},\cdots,d_{Kn}^{(\tau)})$ as $h^{(\tau)}$. Since $h$ is a concave function at $[0,\,+\infty)^K$, we have
\begin{equation}
\begin{aligned}
   h\le h^{(\tau)}+\sum_{k=1}^K \frac{\partial h}{\partial d_k}\bigg |_{n,{\tau}}\cdot(d_{kn}-d_{kn}^{(\tau)}), 
\end{aligned}
\end{equation}
which holds for any $n\in\{1,\cdots,N\}$ and $d_{1n},\cdots, d_{Kn}\in [0,\,+\infty)$. Summing over $n$, it follows that
\begin{multline}
    J(\mathbf{c}_1,\cdots,\mathbf{c}_K) \le \sum_{n=1}^N \bigg[h^{(\tau)}+\sum_{k=1}^K \frac{\partial h}{\partial d_k}\bigg |_{n,{\tau}}\cdot(d_{kn}-d_{kn}^{(\tau)})\bigg]\\
    =J(\mathbf{c}_1^{(\tau)},\cdots,\mathbf{c}_K^{(\tau)})+\frac{1}{2}\sum_{k=1}^K\sum_{n=1}^N \bigg(\frac{\partial h}{\partial d_k}\bigg|_{n,{\tau}}\cdot \big(\|\mathbf{x}_n-\mathbf{c}_k\|_2^2\\
    -\|\mathbf{x}_n-\mathbf{c}_k^{(\tau)}\|_2^2\big)\bigg).
\end{multline}

Denote the function on the right side of the inequality as $M(\mathbf{c}_1,\cdots,\mathbf{c}_K)$. With the boundness condition, we have $\sum_{n=1}^N \frac{\partial h}{\partial d_k}\big|_{n,\tau}>0$ for any $k$ and $\tau$, thus, $M$ is a quadratic function and strictly convex, with the unique global minimizer at $(\mathbf{c}_1^{(\tau+1)},\cdots,\mathbf{c}_K^{(\tau+1)})$ defined by~\eqref{smooth_km_sgd}. Denote $J(\mathbf{c}_1^{(\tau)},\cdots,\mathbf{c}_K^{(\tau)})$ as $J^{(\tau)}$, $M(\mathbf{c}_1^{(\tau)},\cdots,\mathbf{c}_K^{(\tau)})$ as $M^{(\tau)}$, and $\partial_{\mathbf{c}_k} J(\mathbf{c}_1^{(\tau)},\cdots,\mathbf{c}_K^{(\tau)})$ as $\partial_{\mathbf{c}_k} J^{(\tau)}$. Each iteration of the centroid will reduce the objective function by
\begin{equation}
\begin{aligned}
\label{theorem1_proof_inequality}
    &J^{(\tau)}-J^{(\tau+1)}\\
    &\ge J^{(\tau)}-M^{(\tau+1)}\\
    &=\frac{1}{2}\sum_{k=1}^K \sum_{n=1}^N \frac{\partial h}{\partial d_k}\bigg|_{n,\tau}\cdot\big(\|\mathbf{x}_n-\mathbf{c}_k^{(\tau)}\|_2^2-\|\mathbf{x}_n-\mathbf{c}_k^{(\tau+1)}\|_2^2\big)\\
    &\text{substituting~\eqref{smooth_km_sgd}}\\
    &=\frac{1}{2}\sum_{k=1}^K\sum_{n=1}^N \frac{\partial h}{\partial d_k}\bigg|_{n,\tau} \cdot \big( \|\mathbf{x}_n-\mathbf{c}_k^{(\tau)}\|_2^2-\|\mathbf{x}_n-\mathbf{c}_k^{(\tau)}\\
    &+\gamma_k^{(\tau)}\partial_{\mathbf{c}_k}J^{(\tau)}\|_2^2\big)\\
    &=\frac{1}{2}\sum_{k=1}^K\sum_{n=1}^N \frac{\partial h}{\partial d_k}\bigg|_{n,\tau} \cdot \big(-(\gamma_k^{(\tau)})^2\|\partial_{\mathbf{c}_k}J^{(\tau)}\|_2^2\\
    &-2\gamma_k^{(\tau)}(\mathbf{x}_n-\mathbf{c}_k^{(\tau)})\trans \partial_{\mathbf{c}_k}J^{(\tau)}\big)\\
    &\text{using $\gamma_k^{(\tau)}=\frac{1}{\sum_{n=1}^N \frac{\partial h}{\partial d_k}|_{n,\tau}}$}\\
    &\text{and $\partial_{\mathbf{c}_k} J^{(\tau)}=-\sum_{n=1}^N \frac{\partial h}{\partial d_k}\cdot (\mathbf{x}_n-\mathbf{c}_k^{(\tau)})$}\\
    &=\sum_{k=1}^K -\frac{1}{2}\gamma^{(\tau)}_k\|\partial_{\mathbf{c}_k}J^{(\tau)}\|_2^2+\sum_{k=1}^K \gamma_k^{(\tau)}\|\partial_{\mathbf{c}_k}J^{(\tau)}\|_2^2\\
    &=\sum_{k=1}^K \frac{1}{2} \gamma_k^{(\tau)}\|\partial_{\mathbf{c}_k}J^{(\tau)}\|_2^2\\
    &\text{with the boundness condition of $\gamma_k^{(\tau)}$}\\
    &\ge \frac{1}{2}\epsilon\sum_{k=1}^K \|\partial_{\mathbf{c}_k}J^{(\tau)}\|_2^2,
\end{aligned}
\end{equation}

where $\epsilon$ is a positive number. Hence, the sequence $(J^{(1)},J^{(2)},\cdots)$ is non-increasing, and with the boundness condition that $h>-\infty$, we have $\lim_{\tau\to +\infty} (J^{(\tau)}-J^{(\tau+1)})\to 0$. If the left side of the inequality~\eqref{theorem1_proof_inequality} converges to zero, the right side of the inequality also converges to zero since it is non-negative. Consequently, we have $\lim_{\tau\to+\infty} \partial_{\mathbf{c}_k}J^{(\tau)}\to 0$ for all $k$. Therefore, the sequence $(\mathbf{c}_k^{(\tau)})$ converges to a stationary point of the objective function $J$. Because $(J^{(1)},J^{(2)},\cdots)$ is non-increasing, only (local) minimizers or saddle points appear as limit points.

\end{appendices}

\bibliographystyle{IEEEtran}
\bibliography{IEEEabrv,refs}

\begin{thebibliography}{10}
\providecommand{\url}[1]{#1}
\csname url@samestyle\endcsname
\providecommand{\newblock}{\relax}
\providecommand{\bibinfo}[2]{#2}
\providecommand{\BIBentrySTDinterwordspacing}{\spaceskip=0pt\relax}
\providecommand{\BIBentryALTinterwordstretchfactor}{4}
\providecommand{\BIBentryALTinterwordspacing}{\spaceskip=\fontdimen2\font plus
\BIBentryALTinterwordstretchfactor\fontdimen3\font minus \fontdimen4\font\relax}
\providecommand{\BIBforeignlanguage}[2]{{%
\expandafter\ifx\csname l@#1\endcsname\relax
\typeout{** WARNING: IEEEtran.bst: No hyphenation pattern has been}%
\typeout{** loaded for the language `#1'. Using the pattern for}%
\typeout{** the default language instead.}%
\else
\language=\csname l@#1\endcsname
\fi
#2}}
\providecommand{\BIBdecl}{\relax}
\BIBdecl

\bibitem{krawczyk2016learning}
B.~Krawczyk, ``Learning from imbalanced data: open challenges and future directions,'' \emph{Progress in Artificial Intelligence}, vol.~5, no.~4, pp. 221--232, 2016.

\bibitem{he2009learning}
H.~He and E.~A. Garcia, ``Learning from imbalanced data,'' \emph{IEEE Transactions on knowledge and data engineering}, vol.~21, no.~9, pp. 1263--1284, 2009.

\bibitem{ramyachitra2014imbalanced}
D.~Ramyachitra and P.~Manikandan, ``Imbalanced dataset classification and solutions: a review,'' \emph{International Journal of Computing and Business Research (IJCBR)}, vol.~5, no.~4, pp. 1--29, 2014.

\bibitem{tanha2020boosting}
J.~Tanha, Y.~Abdi, N.~Samadi, N.~Razzaghi, and M.~Asadpour, ``Boosting methods for multi-class imbalanced data classification: an experimental review,'' \emph{Journal of Big Data}, vol.~7, pp. 1--47, 2020.

\bibitem{lu2019self}
Y.~Lu, Y.-M. Cheung, and Y.~Y. Tang, ``Self-adaptive multiprototype-based competitive learning approach: A k-means-type algorithm for imbalanced data clustering,'' \emph{IEEE transactions on cybernetics}, vol.~51, no.~3, pp. 1598--1612, 2019.

\bibitem{press2007gaussian}
W.~Press, S.~Teukolsky, W.~Vetterling, and B.~Flannery, ``Gaussian mixture models and k-means clustering,'' \emph{Numerical recipes: the art of scientific computing}, pp. 842--850, 2007.

\bibitem{shireman2017examining}
E.~Shireman, D.~Steinley, and M.~J. Brusco, ``Examining the effect of initialization strategies on the performance of gaussian mixture modeling,'' \emph{Behavior research methods}, vol.~49, pp. 282--293, 2017.

\bibitem{macqueen1967some}
J.~MacQueen \emph{et~al.}, ``Some methods for classification and analysis of multivariate observations,'' in \emph{Proceedings of the fifth Berkeley symposium on mathematical statistics and probability}, vol.~1, no.~14.\hskip 1em plus 0.5em minus 0.4em\relax Oakland, CA, USA, 1967, pp. 281--297.

\bibitem{lloyd1982least}
S.~Lloyd, ``Least squares quantization in pcm,'' \emph{IEEE transactions on information theory}, vol.~28, no.~2, pp. 129--137, 1982.

\bibitem{bezdek2013pattern}
J.~C. Bezdek, \emph{Pattern recognition with fuzzy objective function algorithms}.\hskip 1em plus 0.5em minus 0.4em\relax Springer Science \& Business Media, 2013.

\bibitem{krishnapuram1993possibilistic}
R.~Krishnapuram and J.~M. Keller, ``A possibilistic approach to clustering,'' \emph{IEEE transactions on fuzzy systems}, vol.~1, no.~2, pp. 98--110, 1993.

\bibitem{pal2005possibilistic}
N.~R. Pal, K.~Pal, J.~M. Keller, and J.~C. Bezdek, ``A possibilistic fuzzy c-means clustering algorithm,'' \emph{IEEE transactions on fuzzy systems}, vol.~13, no.~4, pp. 517--530, 2005.

\bibitem{tsai2011fuzzy}
D.-M. Tsai and C.-C. Lin, ``Fuzzy c-means based clustering for linearly and nonlinearly separable data,'' \emph{Pattern recognition}, vol.~44, no.~8, pp. 1750--1760, 2011.

\bibitem{krinidis2010robust}
S.~Krinidis and V.~Chatzis, ``A robust fuzzy local information c-means clustering algorithm,'' \emph{IEEE transactions on image processing}, vol.~19, no.~5, pp. 1328--1337, 2010.

\bibitem{yang2020feature}
M.-S. Yang and J.~B. Benjamin, ``Feature-weighted possibilistic c-means clustering with a feature-reduction framework,'' \emph{IEEE Transactions on Fuzzy Systems}, vol.~29, no.~5, pp. 1093--1106, 2020.

\bibitem{yang2021collaborative}
M.-S. Yang and K.~P. Sinaga, ``Collaborative feature-weighted multi-view fuzzy c-means clustering,'' \emph{Pattern Recognition}, vol. 119, p. 108064, 2021.

\bibitem{zhang2003clustering}
D.-Q. Zhang and S.-C. Chen, ``Clustering incomplete data using kernel-based fuzzy c-means algorithm,'' \emph{Neural processing letters}, vol.~18, pp. 155--162, 2003.

\bibitem{huang2011multiple}
H.-C. Huang, Y.-Y. Chuang, and C.-S. Chen, ``Multiple kernel fuzzy clustering,'' \emph{IEEE Transactions on Fuzzy Systems}, vol.~20, no.~1, pp. 120--134, 2011.

\bibitem{tang2023knowledge}
Y.~Tang, Z.~Pan, X.~Hu, W.~Pedrycz, and R.~Chen, ``Knowledge-induced multiple kernel fuzzy clustering,'' \emph{IEEE Transactions on Pattern Analysis and Machine Intelligence}, 2023.

\bibitem{coates2012learning}
A.~Coates and A.~Y. Ng, ``Learning feature representations with k-means,'' in \emph{Neural Networks: Tricks of the Trade: Second Edition}.\hskip 1em plus 0.5em minus 0.4em\relax Springer, 2012, pp. 561--580.

\bibitem{yang2017towards}
B.~Yang, X.~Fu, N.~D. Sidiropoulos, and M.~Hong, ``Towards k-means-friendly spaces: Simultaneous deep learning and clustering,'' in \emph{international conference on machine learning}.\hskip 1em plus 0.5em minus 0.4em\relax PMLR, 2017, pp. 3861--3870.

\bibitem{caron2018deep}
M.~Caron, P.~Bojanowski, A.~Joulin, and M.~Douze, ``Deep clustering for unsupervised learning of visual features,'' in \emph{Proceedings of the European conference on computer vision (ECCV)}, 2018, pp. 132--149.

\bibitem{fard2020deep}
M.~M. Fard, T.~Thonet, and E.~Gaussier, ``Deep k-means: Jointly clustering with k-means and learning representations,'' \emph{Pattern Recognition Letters}, vol. 138, pp. 185--192, 2020.

\bibitem{xiong2006k}
H.~Xiong, J.~Wu, and J.~Chen, ``K-means clustering versus validation measures: a data distribution perspective,'' in \emph{Proceedings of the 12th ACM SIGKDD international conference on Knowledge discovery and data mining}, 2006, pp. 779--784.

\bibitem{zhou2020effect}
K.~Zhou and S.~Yang, ``Effect of cluster size distribution on clustering: a comparative study of k-means and fuzzy c-means clustering,'' \emph{Pattern Analysis and Applications}, vol.~23, pp. 455--466, 2020.

\bibitem{noordam2002multivariate}
J.~Noordam, W.~Van Den~Broek, and L.~Buydens, ``Multivariate image segmentation with cluster size insensitive fuzzy c-means,'' \emph{Chemometrics and intelligent laboratory systems}, vol.~64, no.~1, pp. 65--78, 2002.

\bibitem{lin2014size}
P.-L. Lin, P.-W. Huang, C.-H. Kuo, and Y.~Lai, ``A size-insensitive integrity-based fuzzy c-means method for data clustering,'' \emph{Pattern Recognition}, vol.~47, no.~5, pp. 2042--2056, 2014.

\bibitem{askari2021fuzzy}
S.~Askari, ``Fuzzy c-means clustering algorithm for data with unequal cluster sizes and contaminated with noise and outliers: Review and development,'' \emph{Expert Systems with Applications}, vol. 165, p. 113856, 2021.

\bibitem{liang2012k}
J.~Liang, L.~Bai, C.~Dang, and F.~Cao, ``The $ k $-means-type algorithms versus imbalanced data distributions,'' \emph{IEEE Transactions on Fuzzy Systems}, vol.~20, no.~4, pp. 728--745, 2012.

\bibitem{zeng2023soft}
S.~Zeng, X.~Duan, J.~Bai, W.~Tao, K.~Hu, and Y.~Tang, ``Soft multi-prototype clustering algorithm via two-layer semi-nmf,'' \emph{IEEE Transactions on Fuzzy Systems}, 2023.

\bibitem{karayiannis1994meca}
N.~B. Karayiannis, ``Meca: Maximum entropy clustering algorithm,'' in \emph{Proceedings of 1994 IEEE 3rd international fuzzy systems conference}.\hskip 1em plus 0.5em minus 0.4em\relax IEEE, 1994, pp. 630--635.

\bibitem{li1995maximum}
R.-P. Li and M.~Mukaidono, ``A maximum-entropy approach to fuzzy clustering,'' in \emph{Proceedings of 1995 IEEE International Conference on Fuzzy Systems.}, vol.~4.\hskip 1em plus 0.5em minus 0.4em\relax IEEE, 1995, pp. 2227--2232.

\bibitem{aloise2009np}
D.~Aloise, A.~Deshpande, P.~Hansen, and P.~Popat, ``Np-hardness of euclidean sum-of-squares clustering,'' \emph{Machine learning}, vol.~75, pp. 245--248, 2009.

\bibitem{wu2012advances}
J.~Wu, \emph{Advances in K-means clustering: a data mining thinking}.\hskip 1em plus 0.5em minus 0.4em\relax Springer Science \& Business Media, 2012.

\bibitem{bottou1994convergence}
L.~Bottou and Y.~Bengio, ``Convergence properties of the k-means algorithms,'' \emph{Advances in neural information processing systems}, vol.~7, 1994.

\bibitem{groll2005new}
L.~Groll and J.~Jakel, ``A new convergence proof of fuzzy c-means,'' \emph{IEEE Transactions on Fuzzy Systems}, vol.~13, no.~5, pp. 717--720, 2005.

\bibitem{gupta2019fuzzy}
A.~Gupta, S.~Datta, and S.~Das, ``Fuzzy clustering to identify clusters at different levels of fuzziness: An evolutionary multiobjective optimization approach,'' \emph{IEEE transactions on cybernetics}, vol.~51, no.~5, pp. 2601--2611, 2019.

\bibitem{asuncion2007uci}
A.~Asuncion and D.~Newman, ``Uci machine learning repository,'' 2007.

\bibitem{arthur2007k}
D.~Arthur and S.~Vassilvitskii, ``K-means++ the advantages of careful seeding,'' in \emph{Proceedings of the eighteenth annual ACM-SIAM symposium on Discrete algorithms}, 2007, pp. 1027--1035.

\bibitem{strehl2002cluster}
A.~Strehl and J.~Ghosh, ``Cluster ensembles---a knowledge reuse framework for combining multiple partitions,'' \emph{Journal of machine learning research}, vol.~3, no. Dec, pp. 583--617, 2002.

\bibitem{yeung2001details}
K.~Y. Yeung and W.~L. Ruzzo, ``Details of the adjusted rand index and clustering algorithms, supplement to the paper an empirical study on principal component analysis for clustering gene expression data,'' \emph{Bioinformatics}, vol.~17, no.~9, pp. 763--774, 2001.

\bibitem{graves2010kernel}
D.~Graves and W.~Pedrycz, ``Kernel-based fuzzy clustering and fuzzy clustering: A comparative experimental study,'' \emph{Fuzzy sets and systems}, vol. 161, no.~4, pp. 522--543, 2010.

\bibitem{huang2012range}
M.~Huang, Z.~Xia, H.~Wang, Q.~Zeng, and Q.~Wang, ``The range of the value for the fuzzifier of the fuzzy c-means algorithm,'' \emph{Pattern Recognition Letters}, vol.~33, no.~16, pp. 2280--2284, 2012.

\bibitem{beyer1999nearest}
K.~Beyer, J.~Goldstein, R.~Ramakrishnan, and U.~Shaft, ``When is “nearest neighbor” meaningful?'' in \emph{Database Theory—ICDT’99: 7th International Conference Jerusalem, Israel, January 10--12, 1999 Proceedings 7}.\hskip 1em plus 0.5em minus 0.4em\relax Springer, 1999, pp. 217--235.

\bibitem{lecun1998mnist}
Y.~LeCun, ``The mnist database of handwritten digits,'' \emph{http://yann. lecun. com/exdb/mnist/}, 1998.

\bibitem{van2008visualizing}
L.~Van~der Maaten and G.~Hinton, ``Visualizing data using t-sne.'' \emph{Journal of machine learning research}, vol.~9, no.~11, 2008.

\end{thebibliography}

\begin{IEEEbiographynophoto}{Yudong He}
received the B.Sc. degree from The University of Science and Technology of China, China, in 2017, and the Ph.D. degree from The Hong Kong University of
Science and Technology in 2022. He is currently a Post-doctoral fellow with the Department
of Industry Engineering and Decision Analytics, The Hong Kong University of
Science and Technology. His research interests include
audio signal processing, compressed sensing, and machine learning.
\end{IEEEbiographynophoto}

\end{document}